\documentclass[11pt,twoside]{article}

\usepackage[
  letterpaper,
  left=0.90in,
  right=0.90in,
  top=0.6in,
  bottom=0.84in,
  headheight=14pt,
  headsep=0.30in,
  footskip=0.34in,
  includeheadfoot
]{geometry}
\usepackage[T1]{fontenc}
\usepackage{amsmath}
\usepackage{amsthm}
\usepackage{newtxtext}
\usepackage{newtxmath}
\usepackage{microtype}
\usepackage{setspace}
\usepackage{mathtools}
\usepackage{bm}
\allowdisplaybreaks

\usepackage{graphicx}
\usepackage{booktabs}
\usepackage{multirow}
\usepackage{array}
\usepackage{tabularx}
\usepackage{float}
\usepackage[font=small,labelfont=bf]{caption}
\usepackage{algorithm}
\usepackage{algpseudocode}

\usepackage{tikz}
\usetikzlibrary{
  positioning,
  arrows.meta,
  fit,
  backgrounds,
  calc,
  shapes.geometric
}

\usepackage{enumitem}
\usepackage[round,authoryear]{natbib}
\usepackage{xcolor}
\usepackage{hyperref}
\usepackage{titlesec}

\titlespacing*{\section}
{0pt}{3.0ex plus 0.8ex minus 0.2ex}{0pt}
\hypersetup{
  colorlinks=true,
  linkcolor=blue,
  citecolor=blue,
  urlcolor=black,
  pdftitle={Two-Timescale Hierarchical Reinforcement Learning for Resilient Operations},
  pdfauthor={Young Hyun Cho, Franz Stoll, Will Wei Sun, Guang Lin, Stephan Biller}
}
\usepackage{fancyhdr}
\fancypagestyle{paperstyle}{%
  \fancyhf{}
  \fancyhead[LE]{\small Cho et al.: Two-Timescale HRL for Resilient Operations}
  \fancyhead[RO]{\small Cho et al.: Two-Timescale HRL for Resilient Operations}
  \fancyfoot[C]{\small\thepage}

}
\fancypagestyle{firstpagestyle}{%
  \fancyhf{}
  \fancyfoot[C]{\small\thepage}

}
\newtheorem{theorem}{Theorem}

\theoremstyle{definition}
\newtheorem{assumption}{Assumption}

\theoremstyle{remark}
\newtheorem{remark}{Remark}

\newcommand{\E}{\mathbb{E}}

\newcommand{\Gap}{\mathcal G}

\newcommand{\WKL}[1]{\mathrm{KL}_{#1}}

\newcommand{\supplementaryname}{Supplementary Material of}
\newcommand{\beginsupplementary}{%
  \clearpage
  \setcounter{section}{0}
  \setcounter{subsection}{0}
  \setcounter{equation}{0}
  \setcounter{figure}{0}
  \setcounter{table}{0}
  \setcounter{theorem}{0}
  \setcounter{proposition}{0}
  \setcounter{lemma}{0}
  \setcounter{corollary}{0}
  \setcounter{assumption}{0}
  \setcounter{definition}{0}
  \setcounter{remark}{0}
  \renewcommand{\thesection}{S\arabic{section}}
  \renewcommand{\thesubsection}{\thesection.\arabic{subsection}}
  \renewcommand{\thesubsubsection}{\thesubsection.\arabic{subsubsection}}
  \renewcommand{\theequation}{S\arabic{equation}}
  \renewcommand{\thefigure}{S\arabic{figure}}
  \renewcommand{\thetable}{S\arabic{table}}
  \renewcommand{\thetheorem}{S\arabic{theorem}}
  \renewcommand{\theproposition}{S\arabic{proposition}}
  \renewcommand{\thelemma}{S\arabic{lemma}}
  \renewcommand{\thecorollary}{S\arabic{corollary}}
  \renewcommand{\theassumption}{S\arabic{assumption}}
  \renewcommand{\thedefinition}{S\arabic{definition}}
  \renewcommand{\theremark}{S\arabic{remark}}
  \renewcommand{\theHsection}{supp.section.\arabic{section}}
  \renewcommand{\theHsubsection}{supp.subsection.\arabic{section}.\arabic{subsection}}
  \renewcommand{\theHsubsubsection}{supp.subsubsection.\arabic{section}.\arabic{subsection}.\arabic{subsubsection}}
  \renewcommand{\theHequation}{supp.equation.\arabic{equation}}
  \renewcommand{\theHfigure}{supp.figure.\arabic{figure}}
  \renewcommand{\theHtable}{supp.table.\arabic{table}}
  \renewcommand{\theHtheorem}{supp.theorem.\arabic{theorem}}
  \renewcommand{\theHproposition}{supp.proposition.\arabic{proposition}}
  \renewcommand{\theHlemma}{supp.lemma.\arabic{lemma}}
  \renewcommand{\theHcorollary}{supp.corollary.\arabic{corollary}}
  \renewcommand{\theHassumption}{supp.assumption.\arabic{assumption}}
  \renewcommand{\theHdefinition}{supp.definition.\arabic{definition}}
  \renewcommand{\theHremark}{supp.remark.\arabic{remark}}
}

\begin{document}
\raggedbottom
\thispagestyle{firstpagestyle}

\begin{center}

{\fontsize{22.5}{27}\selectfont
Two-Timescale Hierarchical Reinforcement Learning\\[-0.05em]
for Resilient Operations\par}

\vspace{1.35em}

{\fontsize{13.5}{16}\selectfont Young Hyun Cho\textsuperscript{*}\par}
\vspace{0.08em}
{\fontsize{9.5}{11.5}\selectfont
Department of Statistics, Harvard University,
\href{mailto:younghyuncho@fas.harvard.edu}{younghyuncho@fas.harvard.edu}\par}

\vspace{0.62em}

{\fontsize{13.5}{16}\selectfont Franz Stoll\par}
\vspace{0.08em}
{\fontsize{9.5}{11.5}\selectfont
Edwardson School of Industrial Engineering, Purdue University,
\href{mailto:stollf@purdue.edu}{stollf@purdue.edu}\par}

\vspace{0.62em}

{\fontsize{13.5}{16}\selectfont Will Wei Sun\par}
\vspace{0.08em}
{\fontsize{9.5}{11.5}\selectfont
Mitch Daniels School of Business, Purdue University,
\href{mailto:sun244@purdue.edu}{sun244@purdue.edu}\par}

\vspace{0.62em}

{\fontsize{13.5}{16}\selectfont Guang Lin\par}
\vspace{0.08em}
{\fontsize{9.5}{11.5}\selectfont
Department of Mathematics and School of Mechanical Engineering,
Purdue University, \href{mailto:guanglin@purdue.edu}{guanglin@purdue.edu}\par}

\vspace{0.62em}

{\fontsize{13.5}{16}\selectfont Stephan Biller\par}
\vspace{0.08em}
\begin{minipage}{0.96\textwidth}
\centering
\fontsize{9.5}{11.5}\selectfont
Edwardson School of Industrial Engineering and Mitch Daniels School of Business, Purdue University;\\[-0.08em]
Dauch Center for the Management of Manufacturing Enterprises,
\href{mailto:sbiller@purdue.edu}{sbiller@purdue.edu}
\end{minipage}

\end{center}

\begingroup
\renewcommand{\thefootnote}{*}
\footnotetext{Most of this work was completed while Young Hyun Cho was at the
Department of Statistics, Purdue University.}
\endgroup


\begin{center}
\begin{minipage}{0.91\textwidth}
\small
\setstretch{1.05}
\setlength{\parindent}{1.45em}
\setlength{\parskip}{0.52pt}

\noindent\textbf{Abstract.}\enspace
Unexpected shocks recur in global operations, requiring decision rules that adapt as market and operating conditions change. Many operational systems also have hierarchical structures in which long-term and short-term decisions pursue a shared objective. We study how hierarchical reinforcement learning can strengthen resilience by adapting these interdependent rules jointly. We develop a two-timescale hierarchical reinforcement learning framework that adapts long-term and short-term policies at their respective time scales. Because the policies are interdependent, we synchronize their updates and prove, to our knowledge, the first convergence guarantees for coupled two-timescale learning. Over \(T\) periods, our policies' average gap from an optimal policy pair is \(O(T^{-1/2})\), improving to \(O(\log T/T)\) when poor decisions produce clearer profit losses. In a used-car case study, inventory replenishment is the long-term decision and customer-arrival pricing the short-term decision. Relative to the strongest partially adaptive benchmark, the framework increases mean profit by 9.2\% under joint demand–supply shocks and by 11.8\% under a prolonged shock scenario, while maintaining a more stable profit trajectory over time. Short-term adaptation addresses routine seasonality and one-sided disruptions by responding immediately to changing conditions. Under joint demand–supply shocks, however, it is insufficient alone; long-term adaptation is also needed to create favorable conditions for short-term decisions. Joint adaptation thus yields higher and more stable profits through disruption and recovery. Because many organizations already use hierarchical planning, the framework strengthens operational resilience without altering existing decision structures.

\vspace{1.0em}

\vspace{0.38em}
\noindent\rule{\linewidth}{0.65pt}
\end{minipage}
\end{center}

\clearpage
\pagestyle{paperstyle}

\IfFileExists{main_body_arxiv.tex}{%
  \input{main_body_arxiv}
}{%
\section{Introduction}\label{sec:Intro}
Unexpected operating shocks are now a recurring feature of global operations. Recent events, including the COVID-19 pandemic \citep{ivanov2020predicting}, the Suez Canal blockage \citep{tran2025costs}, climate-related capacity reductions at the Panama Canal \citep{jung2025panama}, and disruptions around the Strait of Hormuz \citep{unctad2026hormuz}, show how shocks can affect supply, logistics, and capacity. Therefore, decision rules optimized for ordinary operating conditions might perform poorly after conditions shift \citep{snyder2016or,katsaliaki2022supply}. After a shock, the problem is not only that the original rule no longer matches the operating condition. If a rule is repeatedly applied without updates, the mismatch can persist into later operating decisions. This can lead to stockouts, delays, lost sales, or inventory imbalances \citep{chopra2004supply,snyder2016or,katsaliaki2022supply}. As a result, a single unmitigated shock can have both short- and long-term effects on operations.

Resilient operations, therefore, require more than static plans for ordinary conditions \citep{chopra2004supply,kleindorfer2005managing}. Firms need decision rules that update as information arrives, while avoiding overreaction to transient noise \citep{christopher2004building,sheffi2005supply}. This balance is difficult because shocks can be temporary or persistent, and decisions made after a shock can affect later operating outcomes. The resulting problem is to learn decision rules that are stable under ordinary fluctuations and responsive when shocks alter operating conditions.

At the same time, many operating systems have a hierarchical decision structure. Long-term decisions affect operations over a longer horizon. Given a long-term decision, short-term decisions are made repeatedly during the operating period. The resulting period outcome reflects both the long-term and short-term decisions. In manufacturing and supply networks, long-term sourcing commitments interact with base-surge allocation and expedited production decisions \citep{tomlin2006value, allon2010global}. In cloud computing and service operations, capacity sizing operates alongside real-time admission control and pricing \citep{maglaras2003pricing}. In retail and revenue management, inventory replenishment is coupled with dynamic pricing \citep{gallego1994optimal, federgruen1999combined, elmaghraby2003dynamic, biller2005dynamic}. Across these settings, performance depends on how long-term and short-term decisions are combined. This interdependence motivates learning the two decision rules jointly rather than fixing either rule.

These considerations point to reinforcement learning (RL) as a natural tool for adaptive operating decisions. RL improves decision rules through repeated interaction with an operating system and performance feedback observed over time. In the hierarchical setting above, however, the learning problem is not to update a single decision rule. The firm must learn a long-term and a short-term decision rule that operate over different decision horizons and affect a shared performance objective. This motivates a hierarchical RL (HRL) formulation in which the long-term and short-term decision rules are modeled as cooperative policies. Figure~\ref{fig:framework_overview} illustrates this structure. The long-term policy selects long-term commitments that shape the state faced by short-term decisions. The short-term policy adapts within the period, and the resulting period outcome reflects both the long-term and short-term decisions. By updating both policies jointly, the two decision layers adapt together to shocks and other changes in operating conditions while working toward a shared objective.

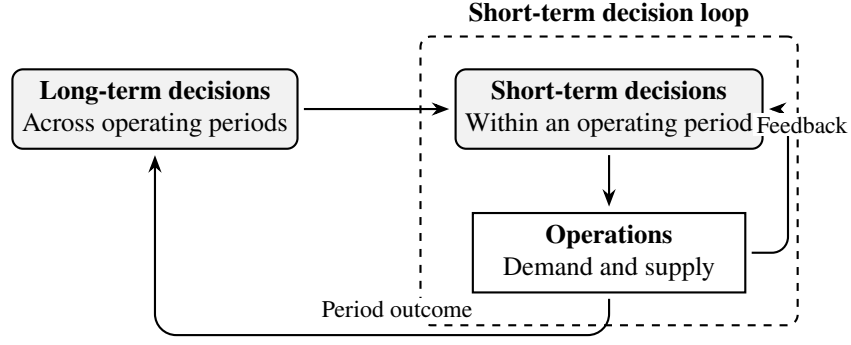
\begin{figure}[t]
\centering
\begin{tikzpicture}[
    >=Stealth,
    every node/.style={font=\small},
    decision/.style={
        rectangle, draw=black, thick, rounded corners,
        fill=gray!10,
        minimum width=3.4cm,
        minimum height=1.05cm,
        align=center
    },
    operation/.style={
        rectangle, draw=black, thick,
        fill=white,
        minimum width=3.6cm,
        minimum height=1.05cm,
        align=center
    },
    loopbox/.style={
        rectangle, draw=black, dashed, thick,
        rounded corners,
        inner xsep=0.45cm,
        inner ysep=0.42cm
    },
    lab/.style={
        font=\footnotesize,
        fill=white,
        inner sep=1.2pt
    }
]

\node[decision]  (slow) at (0,0)
    {\textbf{Long-term decisions}\\Across operating periods};
\node[decision]  (fast) at (6.0,0)
    {\textbf{Short-term decisions}\\Within an operating period};
\node[operation] (env) at (6.0,-1.9)
    {\textbf{Operations}\\Demand and supply};

\begin{scope}[on background layer]
    \node[loopbox, fit=(fast) (env)] (loop) {};
\end{scope}

\node[font=\bfseries\small, fill=white, inner sep=2pt]
    at ([yshift=0.28cm]loop.north) {Short-term decision loop};

\draw[->, thick, shorten >=2pt, shorten <=2pt]
    (slow.east) -- (fast.west);

\draw[->, thick, shorten >=2pt, shorten <=2pt]
    (fast.south) -- (env.north);

\draw[->, thick, rounded corners=6pt, shorten >=2pt, shorten <=2pt]
    (env.east) -- ++(0.55,0) |- (fast.east);
\node[lab] at (8.55,-0.205) {Feedback};

\draw[->, thick, rounded corners=8pt, shorten >=2pt, shorten <=2pt]
    (env.south) -- ++(0,-0.55) -| (slow.south);
\node[lab] at (3.2,-2.65) {Period outcome};

\end{tikzpicture}
\caption{Two-timescale hierarchical decision structure. A long-term decision shapes the operating conditions for short-term decisions. Long-term and short-term decisions jointly determine the period outcome, which informs later long-term decisions.}
\label{fig:framework_overview}
\end{figure}

In this regard, we develop a two-timescale HRL framework for resilient operations. We implement the long-term and short-term policies using proximal policy optimization (PPO), a widely used policy-learning method for updating stochastic policies from sequential interaction data. Standard PPO analyses typically study a single policy update sequence, whereas our two policy updates are coupled because decisions under each policy affect those under the other. We address this coupling by synchronizing their update scales. Our convergence analysis shows that the average gap between our policies and an optimal policy pair is \(O(T^{-1/2})\), improving to \(O(\log T/T)\) when poor decisions produce steep profit losses that provide a clearer learning signal.

We apply the framework to a used-car inventory-pricing setting in which replenishment is the long-term decision and customer-arrival pricing is the short-term decision. We introduce demand and supply shocks separately and jointly and compare HRL with rule-based policies and methods that adapt only one decision layer. Short-term pricing adaptation produces much of the gain under regular operations and one-sided shocks. When demand and supply shocks occur together, however, adapting both decision layers becomes important. HRL improves mean profit over the strongest partially adaptive benchmark by 9.2\% under joint shocks and 11.8\% in the prolonged shock test. Beyond higher aggregate profit, HRL maintains a more stable profit path across operating periods, supporting resilient operations. Further analysis shows that joint learning creates its clearest value during recovery via higher selling prices, stronger inventory retention, and fewer lost sales.

\section{Related Literature and Contributions}

This section positions our work within the related literature and summarizes its contributions. Our work connects resilient operations and joint control, adaptive operational control through RL and HRL with PPO. We discuss these streams in turn, focusing on the need to learn long-term and short-term decision rules jointly for resilient operations. We then summarize the paper's contributions.

\noindent\textbf{Resilient operations and joint control.}
A large body of literature shows that supply disruptions might generate persistent performance losses \citep{hendricks2003effect, hendricks2005association, tang2006robust, tomlin2006value, snyder2016or}. The bullwhip effect literature further shows that demand signals and ordering responses can amplify variability across supply chains, thereby allowing local shocks and forecasting errors to propagate through upstream decisions \citep{lee1997information}. A related stream studies joint operational controls, especially the coupling of inventory, pricing, production, and capacity decisions \citep{federgruen1999combined, chen2003coordinating, elmaghraby2003dynamic, bitran2003overview, biller2005dynamic}. These literatures establish that resilience often depends on both long-term commitments and short-term responsive controls. They leave open how such decision rules can be learned jointly from operating feedback when shocks change the interaction between long-term decisions and short-term responses.

\noindent\textbf{Adaptive operational control.}
RL provides a general framework for improving decision rules through sequential feedback \citep{sutton1998reinforcement}. Recent operational applications include adaptive pricing, resource allocation, assortment optimization, and joint inventory--pricing decisions \citep{miao2021dynamic, perivier2022dynamic, liang2023online, luo2024distribution, chen2024optimal, xu2025joint, wang2025dynamic, kim2025dynamic, luo2026rate}. Multi-agent, multi-echelon, and cooperative learning formulations further highlight the value of coordination in operating systems \citep{swaminathan1998modeling, liu2022multi, leluc2023marlim, zheng2024dual, yang2025multi, kim2024multi, kotecha2025leveraging}. Recent work on foundation models and autonomous agents points to broader AI-enabled supply-chain decision architectures \citep{xu2024multi, menache2025generative, long2026reliability, jannelli2026agentic}. This line of literature shows the value of adaptive and cooperative decision making, but much of it studies one adaptive lever, one decision layer, or coordination across separate agents. Less is known about cooperative two-timescale learning in one operating system, where the long-term and short-term policies share a performance objective and affect the states, outcomes, and data used by each other.

\noindent\textbf{HRL and PPO.} The idea that operational decisions are organized across levels has roots in classical hierarchical production planning \citep{bitran1993hierarchical,hax2012hierarchical}. Modern HRL provides a data-driven language for temporal abstraction, where higher-level decisions shape lower-level control problems \citep{dietterich2000hierarchical, pateria2021hierarchical, barto2003recent}. On the other hand, PPO is a widely used policy-learning method for RL \citep{schulman2017proximal}. Recent theory studies single-policy or single-timescale PPO behavior \citep{shani2020adaptive, jin2023stationary, huang2024ppo}, while less is known about coupled two-timescale PPO learning. Providing convergence theory for this HRL setting is therefore a theoretical contribution in its own right. Motivated by these gaps, our contributions are summarized below:\\
\textbf{$\bullet$ A two-timescale HRL formulation for resilient operations.} As shocks become more frequent, firms need decision rules that adapt across operations. We formulate resilient operations under shocks as a two-timescale HRL problem that jointly learns long-term and short-term policies toward a shared objective. The formulation makes both policies adaptive at their respective decision scales. Because many businesses already organize these decisions hierarchically, the framework is compatible with existing decision structure.\\
\textbf{$\bullet$ Convergence guarantee.} Two-timescale learning raises a distinct difficulty beyond single-policy PPO: each policy interacts with the operating environment while changing the environment of the other. The long-term policy learns from end-of-period outcomes shaped by short-term decisions, while the short-term policy learns under conditions set by long-term decisions. We address this by synchronizing their update and prove, to our knowledge, the first convergence guarantees. Over \(T\) periods, our policies' average gap from an optimal policy pair is \(\mathcal O(T^{-1/2})\). Interestingly, when poor decisions produce steep profit losses under harsh market conditions, those losses strengthen the learning signal and improve this rate to \(\mathcal O(\log T/T)\).\\
\textbf{$\bullet$Implementation and resilient performance in used-car operations.} We implement the framework in a used-car inventory-pricing setting and compare it with rule-based and partially adaptive policies under no shocks and demand, supply, and joint shocks. The framework earns the highest mean profit even without shocks. When demand and supply shocks occur together, as in many real disruptions, it adapts across both decision layers, recovers most strongly, earns the highest profit, and maintains a more stable profit path. The implementation extends to other businesses with analogous decision hierarchies without changing their existing structures.

The remainder of this paper is organized as follows. Section~\ref{sec:setup} formalizes the two-timescale decision problem. Section~\ref{sec:Method} introduces our proposed method. Section~\ref{sec:convergence} develops the convergence analysis and its operational interpretation. Section~\ref{sec:case_study} provides managerial insights from the used-car inventory-pricing case study. Section~\ref{sec:conclusion} concludes the paper. All proofs, controlled numerical illustrations, and case-study implementation details are provided in the Supplementary.

\section{Problem Setup}\label{sec:setup}

We formulate a two-timescale operating problem through long-term and short-term policies. Operating periods are indexed by \(t\ge0\). During period \(t\), the long-term policy chooses a long-term action. Conditional on this long-term action, the system proceeds through \(K_t\) short-term interactions indexed by \(k=0,\ldots,K_t-1\). At each \(k\), the short-term policy chooses a short-term action. Long-term actions can represent replenishment or procurement, and short-term actions can represent pricing or assortment selection.

\subsection{Two-timescale Decision Process}

At the beginning of operating period \(t\), the system state is \(x_t=(I_t,Z_t)\in\mathcal X\). The component \(I_t\) records the physical operating state, such as inventory, pipeline orders, capacity, or backlog. The component \(Z_t\) records operating conditions that affect demand, supply, fulfillment, feasible actions, rewards, or transitions. We model \(Z_t\) as a Markov state component, following Markov-modulated operations models \citep{song1993inventory}. Shocks are represented through changes in \(Z_t\), which alter the operating environment faced by long-term and short-term decisions.

Given \(x_t\), the long-term policy \(\pi(\cdot\mid x_t)\) samples the long-term action \(u_t\in\mathcal U(x_t)\), which remains fixed throughout the operating period and shapes the environment for short-term decisions.

After \(u_t\) is chosen, the within-period state is initialized as \(s_{t,0}\sim P_{\mathrm{init}}(\cdot\mid x_t,u_t)\). We write \(s_{t,k}\) for the within-period Markov information state at short-term interaction \(k\). This state contains the within-period information used for operational decisions, such as remaining inventory, customer information, demand conditions, or fulfillment status. The short-term decision node is \((s_{t,k},u_t,k)\). The long-term action is included because short-term decisions are conditional on the current long-term action. The index \(k\) is included because rewards, feasible actions, and transitions may vary within an operating period.

At short-term interaction \(k=0,\ldots,K_t-1\), the short-term policy selects \(a_{t,k}\sim \phi(\cdot\mid s_{t,k},u_t,k)\), with \(a_{t,k}\in\mathcal A(s_{t,k},u_t,k)\). The system receives reward \(r_{t,k}=r_k(s_{t,k},a_{t,k},u_t)\), and the within-period state evolves according to \(s_{t,k+1}\sim P_{\mathrm f,k}(\cdot\mid s_{t,k},a_{t,k},u_t)\). At the end of the operating period, the next long-term state is drawn as \(x_{t+1}\sim P_{\mathrm s}(\cdot\mid s_{t,K_t},u_t)\). The long-term decision affects the operating environment faced by short-term decisions, and the resulting short-term responses shape the period outcome used for future long-term decisions.

\subsection{Period Payoff and Performance Objective}

Let \(c(x_t,u_t)\) denote the cost tied to the long-term commitment. The period payoff aggregates short-term rewards and subtracts this commitment cost
\[
R_t \triangleq \sum_{k=0}^{K_t-1} r_{t,k}-c(x_t,u_t).
\]
Here \(r_{t,k}\) is a net short-term interaction reward that may include short-term costs or penalties, while \(c(x_t,u_t)\) captures costs tied to the long-term commitment. Thus, \(R_t\) represents the aggregate period payoff generated by the long-term decision and the short-term responses.

For a stationary joint policy \((\pi,\phi)\), the performance criterion is
\[
\mathcal L(\pi,\phi)
\triangleq \E_{\rho_0}^{\pi,\phi}
\left[\sum_{t\ge0}\Gamma^t R_t
\right],
\]
where \(\Gamma\in(0,1)\) is the discount factor and \(\rho_0\) is the initial distribution over long-term states. The expectation is taken under the probability law generated by \(\rho_0\), the policies \((\pi,\phi)\), and the transition kernels. The criterion evaluates both decision layers through their joint contribution to current and future payoffs. It can represent firms that make different decisions at different levels but pursue a shared performance objective.

\section{Methodology}\label{sec:Method}

In this section, we present our method for joint policy learning. We first review the RL and PPO ingredients used to update the two policies. We then present our two-timescale training procedure.

\subsection{Policy Learning with PPO}\label{sec:prelim_rl}

RL learns a policy through repeated interaction with a decision environment, using reward feedback to improve future decisions. Proximal policy optimization (PPO) is a widely used RL algorithm for updating a parameterized stochastic policy \citep{schulman2017proximal}. Using collected interactions, PPO aims to update the policy in a direction that increases expected cumulative reward. To make this update more informative, PPO uses an advantage estimate, which measures whether an observed action performed better or worse than the policy's baseline at the corresponding state.

For a generic policy \(\psi\), the value function \(V^\psi(s)\) is the expected cumulative reward starting from state \(s\) and then following \(\psi\). The action-value function \(Q^\psi(s,a)\) is defined similarly, starting from state \(s\) after taking action \(a\). The advantage is then defined as \(A^\psi(s,a)=Q^\psi(s,a)-V^\psi(s)\), and thus measures the value of action \(a\) relative to the policy's baseline at state \(s\). Given interactions collected under a reference policy \(\psi_{\mathrm{old}}\), PPO updates a candidate policy \(\psi\) through the likelihood ratio \(w_\psi(s,a)= \frac{\psi(a\mid s)}{\psi_{\mathrm{old}}(a\mid s)}\).
Let \(\bar w_\psi(s,a)=\mathrm{clip}(w_\psi(s,a),1-\epsilon,1+\epsilon)\) for a clipping parameter \(\epsilon\in(0,1)\), and define \(m_\psi(s,a)=\min\{w_\psi(s,a)\widehat A(s,a),\bar w_\psi(s,a)\widehat A(s,a)\}\), where \(\widehat A(s,a)\) is an empirical advantage estimate. PPO-Clip maximizes the clipped surrogate objective
\begin{equation}\label{eq:ppo_clip_surrogate}
L^{\mathrm{clip}}(\psi)=\E\!\left[m_\psi(s,a)\right],
\end{equation}
where the expectation is over the state-action pairs collected under \(\psi_{\mathrm{old}}\). The clipping step limits the size of each policy update since a large update in the wrong direction can distort future actions, future states, and later feedback in online operations.

We parameterize the long-term and short-term policies as \(\pi_\theta\) and \(\phi_\omega\), where \(\theta\) and \(\omega\) are the corresponding parameter vectors. We update both through PPO-Clip gradient-based ascent:
\[
\theta\leftarrow\theta+\eta_{\mathrm s}g_{\mathrm s},\qquad
\omega\leftarrow\omega+\eta_{\mathrm f}g_{\mathrm f},
\]
where \(\eta_{\mathrm s}\) and \(\eta_{\mathrm f}\) are the learning rates, and \(g_{\mathrm s}\) and \(g_{\mathrm f}\) are gradient estimates of the corresponding clipped objectives computed from period-level and within-period transitions, respectively. The next subsection gives the full training loop.

\subsection{Two-timescale Hierarchical Reinforcement Learning}
\label{sec:two_timescale_hppo}

We describe how the two policies are updated. At a high level, the training loop alternates between long-term choices at the beginning of an operating period and short-term choices within that period. At the beginning of operating period \(t\), the long-term policy \(\pi_\theta(\cdot\mid x_t)\) samples the long-term action \(u_t\). For \(k=0,\ldots,K_t-1\), the short-term policy \(\phi_\omega(\cdot\mid s_{t,k},u_t,k)\) samples the short-term action \(a_{t,k}\). The short-term update uses within-period records \((s_{t,k},u_t,k,a_{t,k},r_{t,k},s_{t,k+1})\). The long-term update uses period-level records \((x_t,u_t,R_t,x_{t+1})\) after \(R_t\) is realized.

\begin{algorithm}[t]
\caption{Two-timescale hierarchical reinforcement learning}
\label{alg:twotimescale_hppo}
\begin{algorithmic}[1]
\State \textbf{Inputs} operating horizon \(T\), short-term policy batch size \(n_{\mathrm f}\), learning rates \(\eta_{\mathrm s},\eta_{\mathrm f}\), PPO clip parameter \(\epsilon_{\mathrm{clip}}\)
\State \textbf{Initialize} long-term policy \(\pi_{\theta}\) and short-term policy \(\phi_{\omega}\)
\For{\(t=0,1,\dots,T-1\)}
  \State Observe long-term state \(x_t\) and sample \(u_t\sim \pi_{\theta}(\cdot\mid x_t)\)
  \State Initialize within-period state \(s_{t,0}\) given \((x_t,u_t)\)
  \For{each short-term interaction \(k\) during operating period \(t\)}
    \State Sample \(a_{t,k}\sim \phi_{\omega}(\cdot\mid s_{t,k},u_t,k)\)
    \State Observe reward \(r_{t,k}\) and next within-period state \(s_{t,k+1}\)
    \State Add \((s_{t,k},u_t,k,a_{t,k},r_{t,k},s_{t,k+1})\) to the collected short-term data
    \If{the collected short-term data} contains \(n_{\mathrm f}\) records
      \State Update short-term policy \(\phi_{\omega}\) by PPO-Clip with \((\eta_{\mathrm f},\epsilon_{\mathrm{clip}})\), then clear short-term data
    \EndIf
  \EndFor
  \State Compute the period payoff \(R_t\) and observe \(x_{t+1}\)
  \State Update long-term policy \(\pi_{\theta}\) by PPO-Clip with \((\eta_{\mathrm s},\epsilon_{\mathrm{clip}})\)
\EndFor
\State \textbf{Output} learned long-term and short-term policies \(\pi_{\theta}\) and \(\phi_{\omega}\)
\end{algorithmic}
\end{algorithm}

Algorithm~\ref{alg:twotimescale_hppo} combines the two policy updates in a single training loop. The short-term policy is updated after \(n_{\mathrm f}\) within-period records are collected, while the long-term policy is updated after each period payoff is observed. The main tuning inputs are \(n_{\mathrm f}\), \(\eta_{\mathrm s}\), \(\eta_{\mathrm f}\), and \(\epsilon_{\mathrm{clip}}\). Batching limits the influence of any single atypical record, while the learning rates set the update scales and clipping limits the influence of large policy-ratio changes. These controls are important in online learning because an atypical record can create a misleading signal whose effect persists through later interactions.

\begin{remark}[Implementation in Existing Decision Hierarchies]Many businesses already organize long-term and short-term decisions hierarchically. To implement the framework, a firm specifies the state, action, reward, and decision time at each existing layer. The framework adds adaptive policy learning without redesigning this structure.
\end{remark}

On the other hand, Algorithm~\ref{alg:twotimescale_hppo} does not require a specific policy parameterization. The two policies may use different differentiable parameterizations suited to their state and action spaces. Linear parameterizations use fewer parameters and simplify training, while deep neural networks can capture nonlinear relationships but often require more parameters, data, and tuning.

Regardless of policy class, RL updates a policy toward the reward, so reward design is central to what the learned policy optimizes \citep{ng1999policy}. The short-term policy receives \(r_{t,k}\) after each interaction and therefore uses feedback close to its action. The long-term policy receives \(R_t\) only after the period ends. This payoff reflects the long-term action, within-period decisions, realized operating conditions, and shocks. The long-term signal is therefore delayed and depends on both policy layers.

A fixed rule may react to observed states, but it does not learn from new reward feedback. A partially adaptive method updates only one decision layer. Joint training instead updates both policies at their respective scales. The short-term policy adapts within conditions set by long-term decisions, while the long-term policy can change those conditions in later periods. These two adaptation channels provide a mechanism for resilient operations. Section~\ref{sec:case_study} evaluates this mechanism against rule-based and partially adaptive policies.

The two adaptation channels that support resilience also create a coupled learning problem. Short-term updates change the period outcomes used for long-term learning, while long-term updates change the states and operating conditions for short-term learning. The learning rates \(\eta_{\mathrm s}\) and \(\eta_{\mathrm f}\) control the policy-update scales. In sequential operations, each update changes subsequent learning, making the method sensitive to these scales. Large rates can destabilize learning, whereas small rates can delay adaptation. Because each policy changes the other's learning environment, the relative rates are also a joint design choice. Section~\ref{sec:convergence} analyzes a population policy-improvement counterpart and provides a sufficient calibration of the two update scales.

\section{Convergence Analysis}\label{sec:convergence}

This section analyzes the convergence behavior of our method. While Algorithm~\ref{alg:twotimescale_hppo} uses sampled PPO-Clip updates with empirical advantage estimates, we study a population policy-improvement counterpart with exact advantage directions to focus on population-level behavior without sampling noise. This abstraction follows recent analyses of PPO-Clip, where the population policy-improvement step admits an explicit exponentiated-gradient form \citep{huang2024ppo}. We apply this analytical update separately to the long-term and short-term policy updates.

At a decision node, let \(p\) denote the current action distribution, let \(A\) denote the exact advantage direction, and let \(\eta\) denote the learning rate. The representative population update is
\[
p^+(a)
=
\frac{p(a)\exp\{\eta A(a)\}}
{\sum_b p(b)\exp\{\eta A(b)\}}.
\]
The long-term and short-term updates use this rule with learning rates \(\eta_{\mathrm s}\) and \(\eta_{\mathrm f}\), respectively.

\subsection{Analytical Objects and Assumptions}
\label{sec:theory_notation}

Let \((\pi^*,\phi^*)\) be a fixed optimal stationary benchmark for \(\mathcal L\), choosing one if multiple optimal stationary policies exist. For a learning sequence \(\{(\pi_t,\phi_t)\}\), let \(\phi_{t+1}\) denote the short-term policy after the short-term updates in operating period \(t\). We measure progress by the optimality gap \(\Gap_t=\mathcal L(\pi^*,\phi^*)-\mathcal L(\pi_t,\phi_{t+1})\), where the learned policy pair is evaluated after the short-term update in operating period \(t\). For the analysis, we specialize the training loop to a fixed within-period horizon \(K\) and a fixed number \(M\) of short-term policy updates within each operating period.

\begin{assumption}[Markov environment]
\label{ass:markov_env}
The long-term state \(x_t=(I_t,Z_t)\) and the short-term state \(s_{t,k}\) take values in finite sets \(\mathcal X\) and \(\mathcal S\), respectively. The feasible long-term and short-term action sets are finite. The transition dynamics are Markovian at both time scales.
\end{assumption}

\begin{assumption}[Bounded rewards]
\label{ass:bounded_rewards}
There is \(B<\infty\) such that \(|r_k(s,a,u)|\le B\) and \(|c(x,u)|\le B\) for all feasible states and actions.
\end{assumption}

\begin{assumption}[Positive action probabilities]
\label{ass:positive_probs}
Every long-term policy during learning satisfies \(\pi(u\mid x)>0\) for every feasible long-term action \(u\in\mathcal U(x)\). Every short-term policy during learning satisfies \(\phi(a\mid s,u,k)>0\) for every feasible short-term action \(a\in\mathcal A(s,u,k)\).
\end{assumption}

Assumption~\ref{ass:markov_env} imposes a standard Markovian environment, where the transition probability of the next state depends on history only through the current state and action \citep{puterman2014markov}. In our formulation, \(Z_t\) records operating conditions,  including demand shocks, supply shocks, and recovery. Unexpected shocks can therefore be modeled as low-probability transitions from normal operating condition to a disrupted condition. This follows the state-augmentation logic in Markov-modulated operations models \citep{song1993inventory}.

Assumption~\ref{ass:bounded_rewards} keeps the objective \(\mathcal L\) finite. Since each operating period contains \(K\) short-term interactions, the bound \(B\) implies \(|R_t|\le (K+1)B\) for every period. Together with \(\Gamma\in(0,1)\), this gives a finite discounted objective and well-defined optimality gaps. This is the standard bounded-payoff condition used in discounted Markov decision models \citep{puterman2014markov}.

Assumption~\ref{ass:positive_probs} prevents hard elimination of feasible actions during learning. It keeps the KL terms in the convergence proof finite. The condition is natural for stochastic policy classes such as softmax policies, which assign positive probability to feasible actions and update their probabilities during policy improvement \citep{schulman2015trust,schulman2017proximal,shani2020adaptive}.

\subsection{Global Convergence}

We present a global convergence analysis for coupled two-timescale PPO-style learning, which is, to the best of our knowledge, the first result of this type. Recent PPO theory has begun to establish convergence guarantees for single-policy updates \citep{jin2023stationary,huang2024ppo}, but our result is not a direct application of that theory. The hierarchical structure creates a stability issue even under stationary transition dynamics, because the long-term and short-term updates enter each other's learning signals. We tackle this issue by aligning the long-term update scale with the cumulative short-term update scale through the learning rates, as formalized in Theorem~\ref{thm:global_convergence_generic}:

\begin{theorem}[Global average-gap convergence]
\label{thm:global_convergence_generic}
Suppose Assumptions~\ref{ass:markov_env}--\ref{ass:positive_probs} hold. If \(\eta_{\mathrm s}=T^{-1/2}\) and \(M\eta_{\mathrm f}=\{K/(1-\Gamma)\}\eta_{\mathrm s}\), then there exist finite constants \(C_0\) and \(C_1\), independent of \(T\), such that
\[
\frac{1}{T}\sum_{t=0}^{T-1}\E[\Gap_t]
\le
\frac{C_0+C_1}{\sqrt T}.
\]
\end{theorem}

Theorem~\ref{thm:global_convergence_generic} gives an \(O(T^{-1/2})\) bound on the average optimality gap by synchronizing the two learning rates. This synchronization arises since the long-term and short-term updates enter each other's learning signals. Specifically, \(M\eta_{\mathrm f}=\{K/(1-\Gamma)\}\eta_{\mathrm s}\) links the cumulative short-term update scale within one period to the long-term update scale. Here \(M\) appears because the short-term policy is updated \(M\) times within an operating period, so \(M\eta_{\mathrm f}\) is the cumulative short-term update scale over that period. The factor \(K\) appears because each operating period contains \(K\) short-term interactions, and \(1/(1-\Gamma)\) appears because short-term responses enter the discounted objective. Thus, \(K\) and \(\Gamma\) determine how strongly the short-term layer enters the scale of long-term learning.

The synchronization also provides useful managerial insights. When \(K\) is large, short-term decisions account for a larger share of the period payoff \(R_t\), so the long-term policy learns from a signal that depends more heavily on the stability of short-term responses. When \(\Gamma\) is close to one, short-term responses have a larger role in the long-run objective, so their effects persist more strongly in the learning signal. In systems with many within-period interactions or long-run consequences, rich short-term feedback supports more frequent short-term updates. Increasing \(M\), however, calls for reducing \(\eta_{\mathrm f}\), so that \(M\eta_{\mathrm f}\) remains calibrated to \(\eta_{\mathrm s}\). If \(M\eta_{\mathrm f}\) is too large, the period payoff may reflect changes in the short-term response rule rather than a stable evaluation of the long-term action. The calibration keeps short-term learning active while preserving the period payoff as a useful signal for updating future long-term actions.

Beyond the rate, the constants \(C_0\) and \(C_1\) capture factors that affect the size of the bound. The constant \(C_0\) comes from the distance between the initial policy pair and the optimal \((\pi^*,\phi^*)\), so it is larger when the initial long-term and short-term policies are farther from \((\pi^*,\phi^*)\). The constant \(C_1\) captures features of the operating setting. Across settings, severe or persistent shock dynamics can be reflected in \(C_1\) through larger payoff losses from poor decisions, tighter feasible action sets, or transition dynamics that make future states more sensitive to long-term and short-term actions. In our formulation, these effects enter through \(Z_t\), which changes rewards, feasible actions, and transition probabilities. Overall, factors that naturally increase operating loss affect the constant \(C_1\), rather than the \(T^{-1/2}\) order.

\subsection{Accelerated Convergence under Market-Specific Structures}
\label{sec:accelerated_convergence}

Beyond the baseline convergence result in Theorem~\ref{thm:global_convergence_generic}, we next refine this result under a market condition that makes poor decisions visible in the objective. For the following assumption, let \(\Phi(\pi,\phi)\) denote the KL-based policy distance from the optimal stationary pair \((\pi^*,\phi^*)\), with the exact formula given in Section~\ref{app:proof_theorem1} of the Supplementary.

\begin{assumption}[Market sharpness]
\label{ass:market_sharpness}
There exists \(\mu>0\) such that every learning iterate satisfies \(\mathcal L(\pi^*,\phi^*)-\mathcal L(\pi_t,\phi_t)\ge \mu\Phi(\pi_t,\phi_t)\).
\end{assumption}

Assumption~\ref{ass:market_sharpness} links the KL-based policy distance \(\Phi\) to the optimality gap. If the learning policy pair is farther from \((\pi^*,\phi^*)\) in \(\Phi\), then its objective value must be lower by at least a proportional amount. In learning terms, the objective provides a clear loss signal when the current long-term or short-term policy moves in the wrong direction. In market terms, poor replenishment or pricing decisions quickly appear as lost sales or margin loss. A simple deterministic pricing example gives the basic intuition. If \(D(p)=a-bp\) with \(b>0\) and the interior feasible maximizer is \(p^*=a/(2b)\), then \(r(p)=pD(p)\) satisfies \(r(p^*)-r(p)=b(p-p^*)^2\) on the relevant feasible price range \citep{gallego1994optimal}. Assumption~\ref{ass:market_sharpness} is the policy-level analogue of this visible-loss condition.

Assumption~\ref{ass:market_sharpness} does not imply that markets with disruptions or tight constraints have lower absolute loss. It rules out broad flat regions in which policies far from \((\pi^*,\phi^*)\) have nearly optimal objective values. The accelerated result below uses this visible-loss structure through \(\mu\), while operating difficulty remains in the constants.

\begin{theorem}[Accelerated convergence under market sharpness]
\label{thm:rate_improve_average}
Suppose Assumptions~\ref{ass:markov_env}--\ref{ass:market_sharpness} hold. Choose \(\eta_{\mathrm s,t}=2/\{\mu(t+t_0)\}\), with \(t_0\) sufficiently large, and synchronize the short-term learning rate by \(M\eta_{\mathrm f,t}=\{K/(1-\Gamma)\}\eta_{\mathrm s,t}\). Then there exists a finite constant \(C_2\), independent of \(T\), such that for every \(T\ge2\),
\[
\frac{1}{T}\sum_{t=0}^{T-1}\E[\Gap_t]
\le
\frac{C_2\log T}{\mu T}.
\]
\end{theorem}

Theorem~\ref{thm:rate_improve_average} improves the baseline \(T^{-1/2}\) bound to a \(\log T/T\) bound for the average optimality gap. Under Assumption~\ref{ass:market_sharpness}, the parameter \(\mu\) measures how strongly the KL-based policy distance appears as an optimality gap. As \(\mu\) increases, poor long-term or short-term decisions are more sharply separated from good ones in the objective, which gives a stronger contraction. At the same time, the learning-rate synchronization remains necessary because the short-term policy still moves within each operating period and its cumulative movement must remain aligned with long-term learning.

Accordingly, the constant \(C_2\) absorbs initialization, the fixed difficulty terms from Theorem~\ref{thm:global_convergence_generic}, and the correction caused by evaluating \(\Gap_t\) after the short-term update. Severe shock dynamics or tight feasible action sets can increase \(C_2\) through larger payoff losses or transition dynamics that make future states more sensitive to long-term and short-term actions. Thus, \(\mu\) governs the rate improvement, while operating difficulty affects the constant.

\begin{remark}[Market Difficulty and Learning] Harsh market conditions can reduce profit even under an optimal policy pair. Theorem~\ref{thm:rate_improve_average} does not imply that such conditions improve absolute performance. Instead, when poor decisions produce steeper profit losses, those losses provide a clearer learning signal, allowing our policies to distinguish better decisions and approach the optimal policy pair faster. By contrast, a fixed rule does not learn from these losses and may continue to make the same decisions. Thus, although difficult conditions affect all policies, our framework uses the resulting feedback to move more quickly toward the best available response.
\end{remark}

\section{Case Study: Used-Car Inventory and Pricing}
\label{sec:case_study}

We study a used-car retail setting in which replenishment determines the inventory available for future customer arrivals, while pricing determines the margin and purchase probability at each arrival. The case study compares four policy configurations under regular operations, one-sided shocks, and joint demand--supply shocks. The configurations differ in whether adaptation occurs in pricing, replenishment, both, or neither. This comparison allows us to assess how adaptation at each decision layer affects performance during shocks and recovery.

\subsection{Experiment Setup}

The experiment combines heterogeneous demand, seasonality, and demand and supply shocks. We evaluate all policies on matched operating paths and compare them by period profit. The following subsections define the profit and training rewards, the shock channels, and the policy choices.

\subsubsection{Profit and Reward Construction}
\label{sec:reward_construction}

We define the profit for performance comparisons and the two rewards used for policy training. The profit measure aggregates sales, inventory holding, ordering, and lost sales into period profit. The rewards are separated because the long-term and short-term policies make decisions at different time scales and use rewards tied to those decisions. The long-term replenishment policy uses period profit, while the short-term pricing policy uses a customer-arrival reward based on margin, inventory pressure, and stockout events.

Let \(\mathcal C=\{\mathrm{budget},\mathrm{mid},\mathrm{premium}\}\) denote the set of vehicle classes. The period profit is defined as $\Pi_t
=
\sum_{c\in\mathcal C}\sum_{i\in\mathcal S_{c,t}}(p_i-w_c)
-\sum_{c\in\mathcal C}h_c I^+_{c,t}
-F\mathbf 1\!\left\{\sum_{c\in\mathcal C}q_{c,t}>0\right\}
-\sum_{c\in\mathcal C}\lambda_c\ell_{c,t}.
$
Here, \(\mathcal S_{c,t}\) is the set of class-\(c\) units sold in period \(t\), \(p_i\) is the realized selling price of unit \(i\), and \(w_c\) is the acquisition cost for class \(c\). The term \(h_c\) is the per-period holding cost, \(I^+_{c,t}\) is post-sales inventory, \(F\) is the fixed ordering cost, \(q_{c,t}\) is the order quantity, \(\lambda_c\) is the lost-sale penalty, and \(\ell_{c,t}\) is the number of lost sales. Acquisition cost is charged when a unit is sold rather than when it is ordered, so that period profit is aligned with realized sales rather than procurement timing.

\begin{table}[ht!]
\centering
\caption{Rewards used for policy training in the case study.}
\label{tab:reward_construction}
\small
\begin{tabular}{p{0.28\textwidth}p{0.62\textwidth}}
\toprule
\textbf{Reward} & \textbf{Definition} \\
\midrule
Long-term policy reward &
\(\displaystyle r_t^{\mathrm{LT}}=\Pi_t/\kappa\) \\[1.4ex]

Short-term policy reward &
\(\displaystyle
\begin{aligned}
r_n^{\mathrm{ST}}
={}
(p_c-w_c)\mathbf 1\{\mathrm{sale}\}
-\lambda_I\sum_{c\in\mathcal C}\frac{h_c}{N_t}I_c 
-\lambda_{\mathrm{lost}}\mathbf 1\{\mathrm{stockout}\}
\end{aligned}
\) \\
\bottomrule
\end{tabular}
\end{table}

In Table~\ref{tab:reward_construction}, \(\kappa>0\) is a constant used only to normalize the long-term policy reward for training. In the short-term policy reward, \(p_c\) is the posted price for class \(c\), \(N_t\) is the number of customer arrivals in period \(t\), \(I_c\) is the current inventory of class \(c\) at the customer arrival, and \(\lambda_I\) and \(\lambda_{\mathrm{lost}}\) are weights to prevent policies from keeping excessive stock while waiting for high-margin sales and to discourage overly aggressive pricing when inventory is scarce. We highlight that the profit measure \(\Pi_t\) is used for performance comparisons, while \(r_t^{\mathrm{LT}}\) and \(r_n^{\mathrm{ST}}\) are used only for training. The two rewards are not added to \(\Pi_t\), so the profit measure does not double-count the components.

\subsubsection{Shock Setup}
\label{sec:shock_setup}

The experiments use two disruption channels: demand and supply. A demand shock changes arrival volume and customer class mix. Demand surges raise arrivals and may shift demand toward higher-value classes, while contractions do the opposite. A supply shock constrains replenishment by reducing the number of fulfilled orders and delaying inventory availability. Joint shocks activate both channels.

Shocks are exogenous and are not observed before they occur. Once realized, their effects enter the operating state through arrivals, class mix, inventory availability, and order fulfillment. Within each seed, all policies face the same customer stream, purchase draws, shock timing, and supply randomness. Thus, the comparisons are paired, and only the policy differs.

\subsubsection{Baseline Policies and Comparison Design}
\label{sec:policies-comparison}

Table~\ref{tab:factorial1} summarizes the paired \(2\times2\) comparison design. Rows vary the replenishment policy, and columns vary the pricing policy. This design separates the value of RL pricing, RL replenishment, and the joint policy.

\begin{table}[ht!]
\centering
\caption{Paired \(2\times2\) comparison design.}
\label{tab:factorial1}
\small
\begin{tabular}{@{}lcc@{}}
\toprule
 & \textbf{Fixed pricing} & \textbf{RL pricing} \\
\midrule
\textbf{OUL replenishment}
  & OUL+Fixed & OUL+RL \\[2pt]
\textbf{RL replenishment}
  & RL+Fixed  & HRL    \\
\bottomrule
\end{tabular}
\end{table}

OUL stands for order-up-to-level, a standard periodic-review inventory rule rooted in classical inventory theory \citep{arrow1951optimal,arrow1958studies,clark1960optimal,federgruen1984efficient,porteus2002foundations}. In our implementation, the OUL parameters are periodically re-estimated from observations. Thus, OUL is a rule-based benchmark that updates its parameters over time, but its decision rule is prescribed rather than learned from operating feedback. Fixed pricing uses a constant markup over the acquisition cost, while RL pricing updates the customer-arrival pricing policy based on reward feedback.

OUL+Fixed combines OUL with fixed-markup pricing. OUL+RL keeps OUL and replaces fixed pricing with the RL pricing policy. RL+Fixed replaces OUL with the RL replenishment policy and keeps fixed-markup pricing. HRL updates both the replenishment and pricing policies jointly. The four policies therefore compare pricing-only adaptation, replenishment-only adaptation, and joint replenishment-pricing adaptation against a rule-based benchmark.

Because policies are evaluated on matched exogenous operating paths within each seed, profit differences isolate policy behavior. Let \(P(\cdot)\) denote mean profit over a given analysis window. We use the \(2\times2\) interaction contrast \(\Delta_{\mathrm{coord}}
=P(\mathrm{HRL})-P(\mathrm{OUL{+}RL})-P(\mathrm{RL{+}Fixed})+P(\mathrm{OUL{+}Fixed})\) to measure the additional value of adapting replenishment and pricing jointly. A positive \(\Delta_{\mathrm{coord}}\) means that HRL creates more value than the additive benchmark formed by pricing-only and replenishment-only adaptation. A zero value indicates additive effects, and a negative value indicates that joint adaptation does not preserve the separate gains. We compute this contrast over the full evaluation horizon and over disruption and recovery windows.

\subsection{Experiment 1: Multiple Shock Combinations}
\label{sec:results}

Experiment~1 compares the four policies across shock settings. Each policy is evaluated over \(3{,}000\) operating periods and 30 random seeds under the paired comparison design. We organize the analysis in four steps. First, aggregate profitability ranks the policies across shock settings. Second, an event-window analysis compares pre-shock, end-of-shock, and recovery profits to quantify shock-period profit deterioration and post-shock rebound. Third, for joint shocks, price, inventory, and lost-sales patterns diagnose the profit differences. Finally, a \(2\times2\) contrast evaluates whether joint replenishment-pricing learning adds value beyond standalone pricing and replenishment learning. Together, these analyses ask: \emph{which policy earns the highest profit, how profits deteriorate during shocks and rebound afterward, what price, inventory, and lost-sales patterns explain the differences, and whether joint learning adds value beyond standalone pricing and replenishment learning}.

\subsubsection{Aggregate Profitability}
\label{sec:exp1-profit}

We begin with the primary aggregate-profitability question: \emph{which policy earns the highest profit, and how does the ranking change across shock settings?} Table~\ref{tab:exp1-profit} compares mean profit and cumulative profit across the four operating settings. Two patterns are immediate. First, policies with RL pricing achieve substantially higher profit than fixed-pricing policies in every setting. In our setting, pricing is the short-term decision, so this pattern reflects the value of adapting at the faster decision scale. The pricing policy can respond at each customer arrival as demand, supply, and inventory availability change, whereas fixed pricing remains tied to acquisition cost within a prescribed rule. Second, HRL earns the highest mean profit in all cases except supply-only shocks. Under no shocks and demand-only shocks, HRL and OUL+RL are both in the high-profit range. In the supply-only exception, OUL+RL earns the highest mean profit, but HRL remains close based on the 95\% confidence intervals.

\begin{table}[t!]
  \centering
  \caption{Experiment~1: Mean profit by policy and shock setting over periods 450--3{,}000, based on 30 seeds. Values are reported with 95\% confidence intervals. The best policy in each shock setting is in bold.}
  \label{tab:exp1-profit}
  \small
  \begin{tabular}{l l r@{\;\,$\pm$\;\,}l r@{\;\,$\pm$\;\,}l}
    \toprule
    Shock setting & Policy
      & \multicolumn{2}{c}{Profit per period (\$K)}
      & \multicolumn{2}{c}{Cumulative profit (\$M)} \\
    \midrule
    \multirow{4}{*}{No shocks}
      & OUL+Fixed  & 204.1 & 0.1  & 520.5 & 0.4 \\
      & OUL+RL     & 254.2 & 15.7 & 648.3 & 20.2 \\
      & RL+Fixed   & 206.6 & 0.3  & 526.8 & 0.9 \\
      & HRL        & \textbf{263.8} & 14.9 & \textbf{672.8} & 18.0 \\
    \addlinespace
    \multirow{4}{*}{Demand only}
      & OUL+Fixed  & 202.8 & 0.1  & 517.2 & 0.4 \\
      & OUL+RL     & 261.7 & 13.5 & 667.2 & 14.4 \\
      & RL+Fixed   & 205.0 & 0.4  & 522.9 & 0.9 \\
      & HRL        & \textbf{265.5} & 17.2 & \textbf{677.1} & 23.9 \\
    \addlinespace
    \multirow{4}{*}{Supply only}
      & OUL+Fixed  & 184.1 & 0.2  & 469.5 & 0.6 \\
      & OUL+RL     & \textbf{246.4} & 12.7 & \textbf{628.2} & 32.5 \\
      & RL+Fixed   & 188.3 & 0.3  & 480.1 & 0.8 \\
      & HRL        & 238.9 & 17.5 & 609.2 & 24.5 \\
    \addlinespace
    \multirow{4}{*}{Joint shocks}
      & OUL+Fixed  & 183.2 & 0.2  & 467.3 & 0.5 \\
      & OUL+RL     & 238.7 & 12.4 & 608.7 & 31.6 \\
      & RL+Fixed   & 188.2 & 0.3  & 480.0 & 0.7 \\
      & HRL        & \textbf{260.7} & 16.7 & \textbf{664.7} & 22.6 \\
    \bottomrule
  \end{tabular}
\end{table}

The asymmetric gains from the two partially adaptive policies further identify the short-term policy as the primary adaptation lever. OUL+RL substantially improves over OUL+Fixed in every setting, whereas RL+Fixed improves only modestly. In this case study, pricing decisions are made at customer arrivals and receive frequent purchase and profit feedback. As a result, this short decision cycle allows the policy to adapt promptly to shocks, seasonality, and other changes in conditions. Thus, the key mechanism is not pricing itself, but the speed of short-term decisions and feedback.

However, we highlight that short-term adaptation is not the only valuable lever. Under joint shocks, HRL exceeds OUL+RL by about \$22K per period and \$56M in cumulative profit. In this setting, pricing can manage the available inventory, but it cannot change the inventory available to future arrivals. Consequently, adapting replenishment becomes important when shocks affect both demand and supply. This distinction matters because economic shocks and pandemics can combine demand and supply changes in unanticipated ways. By adapting both decision layers, HRL can respond within the current period, adjust inventory for later periods, and support recovery.
\begin{figure}[!t]
\centering\includegraphics[width=0.95\linewidth]{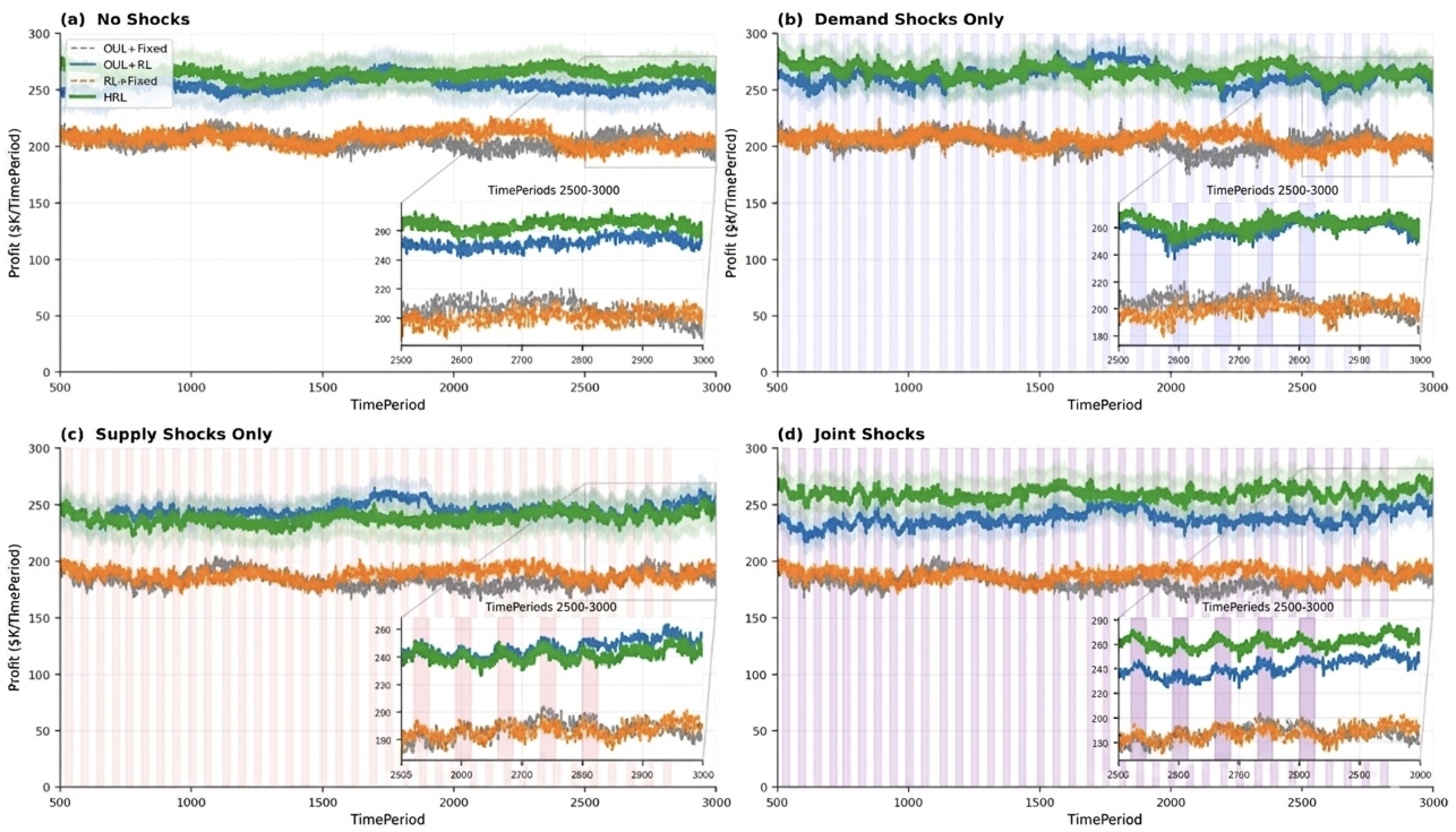}
  \caption{Experiment~1: Profit by shock setting, shown as 52-period rolling means. Panels correspond to (a) no shocks, (b) demand-only shocks, (c) supply-only shocks, and (d) joint shocks. Shaded bands show 95\% confidence intervals across 30 seeds, and colored regions mark shock periods. Insets provide a closer view of periods 2{,}500--3{,}000.}
  \label{fig:exp1-weekly}
\end{figure}

Figure~\ref{fig:exp1-weekly} shows that these aggregate patterns persist across operating periods rather than being driven by a few isolated periods. Under no shocks and demand-only shocks, HRL and OUL+RL remain in the high-profit range. Under supply-only shocks, OUL+RL maintains the higher path over much of the horizon, consistent with the limited immediate role of replenishment learning when order fulfillment is restricted. Under joint demand--supply shocks, HRL maintains the clearest separation, consistent with the value of adapting pricing and replenishment together.

Resilient operations require both strong performance and stability as conditions change. Across the four settings, HRL's profit path varies least across periods. This stability, together with its strong aggregate profit, shows how joint adaptation supports resilient operations.

The results also suggest a practical deployment strategy. Short-term adaptation is the primary response to changing conditions. If a firm expects only regular variation or one-sided shocks, or introduces RL one layer at a time, it can prioritize the short-term policy. However, real disruptions are difficult to anticipate and can affect demand and supply together. These results therefore favor HRL when resilience to complex shocks is the operating objective. The next analysis separates disruption and recovery windows to identify when this advantage emerges.

\subsubsection{Disruption Response}
\label{sec:exp1-shocks}

We next examine event windows around each shock to compare profits before the shock, at its end, and through recovery. 

For each shock event, let \(S_0\) be the mean profit over the five periods preceding the shock, \(S_1\) be the mean profit over the final five periods of the shock, \(R_0\) be the mean profit over the first five recovery periods, and \(R_1\) be the mean profit over the final five periods of the recovery window. We define resistance as \(S_1-S_0\), where lower values indicate larger shock-period losses, and recovery as \(R_1-R_0\), where larger values indicate stronger rebound. These windows are fixed across policies and shock settings.

Figure~\ref{fig:exp1-resilience} reports the four event-window profits for demand-only, supply-only, and joint shocks. Across all three panels, HRL has the largest rebound and the highest late-recovery profit. We examine each shock setting in turn.

\begin{figure}[htbp!]
  \centering
  \includegraphics[width=\linewidth]{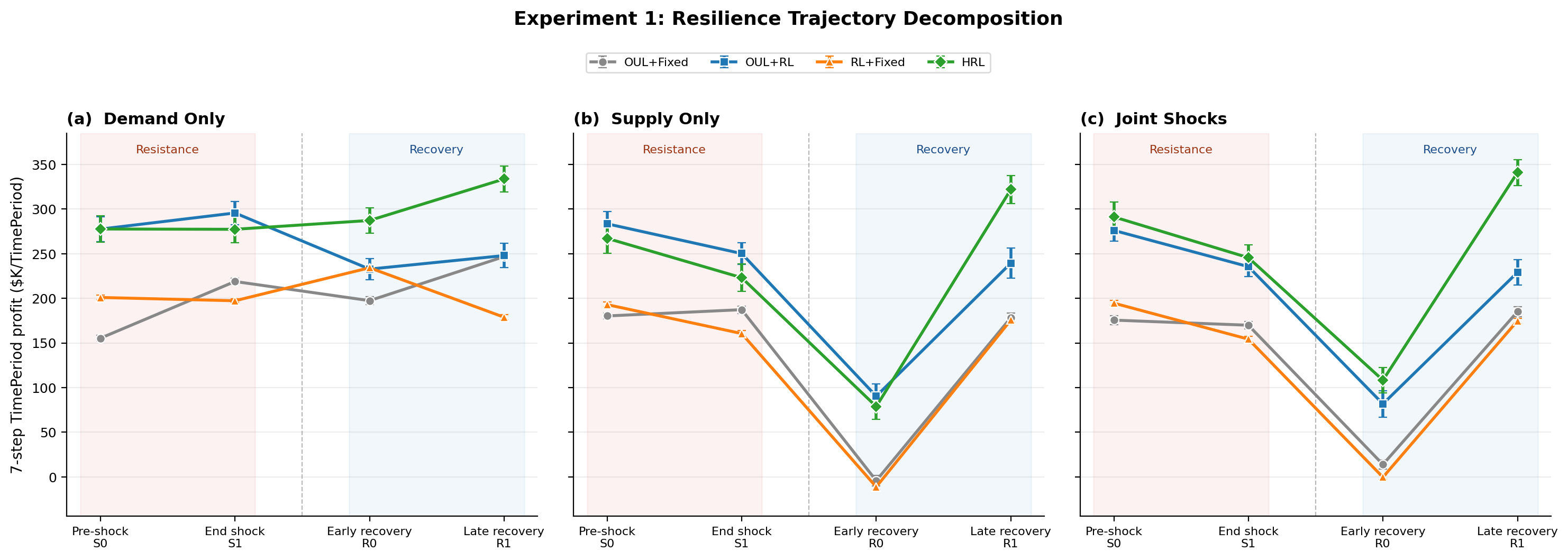}
  \caption{Profit response around disruption events in Experiment~1. Panels~(a)--(c) show profit per period before, during, and after demand-only, supply-only, and joint shocks. Each point reports mean profit over a five-period window. Error bars show 95\% confidence intervals across 30 seeds.}
  \label{fig:exp1-resilience}
\end{figure}

Panel~(a) isolates demand-only shocks, which change customer arrivals and class mix. HRL remains near \$278K per period from \(S_0\) to \(S_1\) and then rises to \$334K by \(R_1\), the highest late-recovery profit. OUL+Fixed also gains during the shock, but this reflects additional customer arrivals rather than policy adjustment. In contrast, RL+Fixed loses \$55K per period from \(R_0\) to \(R_1\). This pattern is consistent with pricing serving as the direct adaptation channel when the shock changes demand.

Panel~(b) isolates supply-only shocks. Because the supply shock restricts fulfilled orders, changing replenishment has limited immediate influence on realized inventory. OUL+RL has the highest mean profit at \(S_1\), reaching \$250K per period compared with \$223K for HRL. During the restriction, pricing can still adjust margins and purchase probabilities for the available inventory. However, HRL rebounds most strongly after the shock ends and reaches \$322K per period at \(R_1\), the highest late-recovery profit. Once the restriction ends, replenishment can again change inventory for future arrivals, allowing HRL to use both adaptation channels during recovery.

Panel~(c) combines demand and supply shocks. Consistent with the aggregate results, HRL has its clearest advantage when both shocks occur together. HRL has the highest profit at \(S_1\), reaching \$245K per period, and remains above OUL+RL at \(R_0\). It then rises from \$109K per period at \(R_0\) to \$341K at \(R_1\), giving the largest rebound and the highest late-recovery profit. The demand shock changes arrivals and class mix, while the supply shock changes future inventory availability. Pricing adjusts margins and purchase probabilities for current inventory, while replenishment changes the inventory available to future arrivals. This pattern is consistent with joint adaptation managing both the use of current inventory and the inventory available during recovery.

Taken together, these results therefore show that HRL is not only a high-profit policy, but also a recovery-oriented policy. Under isolated shocks, dealers may obtain much of the benefit from dynamic pricing because inventory availability remains relatively flexible under demand shocks, while pricing can help protect margins and slow inventory depletion under supply shocks. However, when demand and supply disruptions occur together, pricing flexibility alone is not enough. Dealers also need to rebuild the right inventory mix for future sales while adjusting prices during and after the shock to limit lost sales and margin decline. HRL is most useful in this setting because it coordinates both actions: pricing protects short-run profit during the shock, while replenishment supports a stronger recovery after the shock ends.

\subsubsection{Detailed Diagnosis under Joint Shocks}
\label{sec:exp1-mechanisms}

The preceding analyses show that HRL’s profit and recovery advantages are clearest under joint shocks. We examine now the underlying channels using average selling price, total inventory, and lost sales.

\begin{figure}[!h]
\centering
\includegraphics[width=0.99\linewidth]{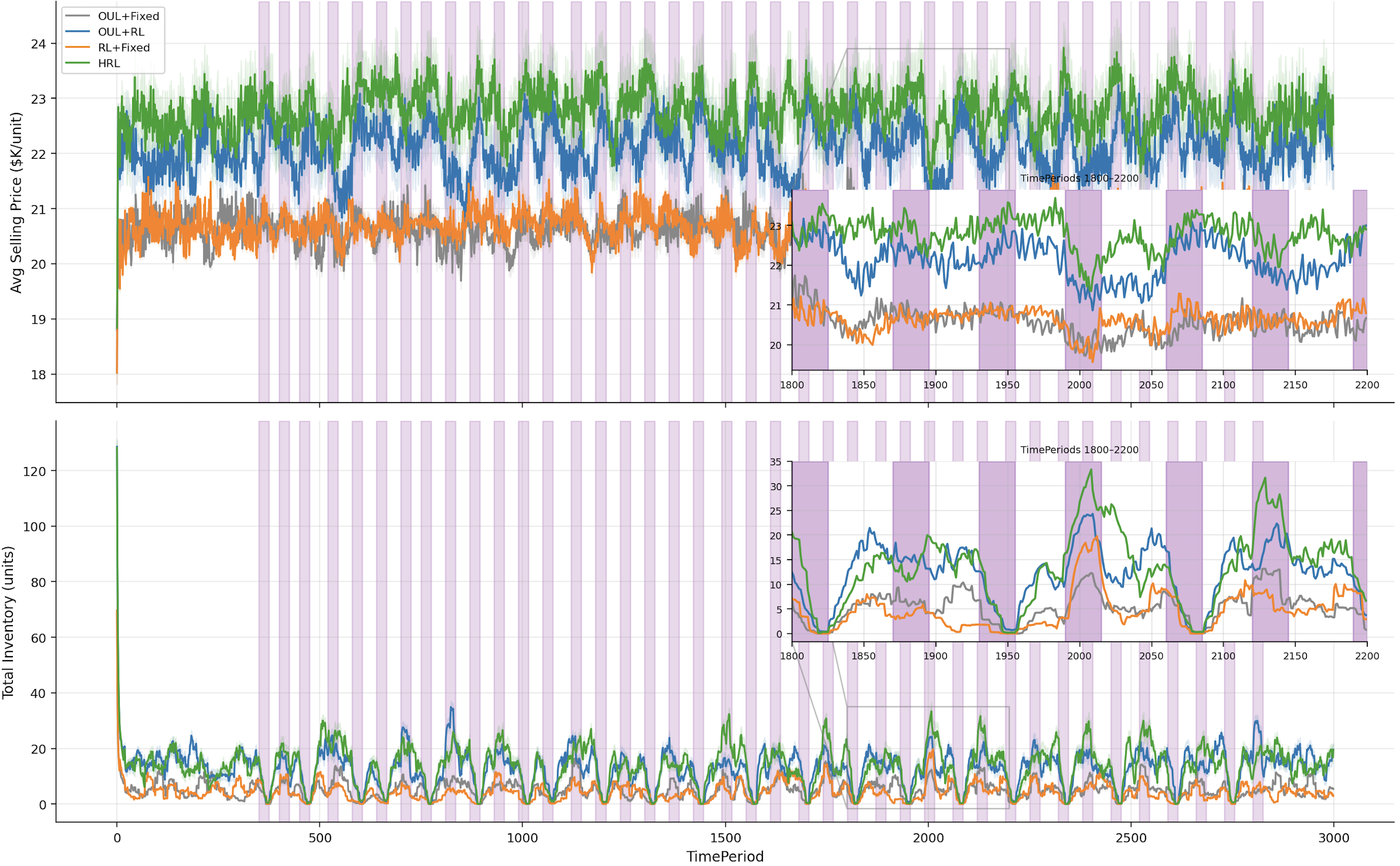}
\caption{Average selling price and total inventory over the full horizon. The top panel reports average selling price, and the bottom panel reports total inventory. Insets show periods 1{,}800--2{,}200.}
\label{fig:exp1-price-inv}
\end{figure}

Figure~\ref{fig:exp1-price-inv} reports full-horizon selling-price and inventory paths under joint shocks. The top panel shows that HRL maintains the highest average selling price for much of the horizon, mostly between \$22.5K and \$23.5K per unit. OUL+RL follows at a lower level, while the fixed-pricing policies remain near \$20K to \$21K per unit. Thus, policies with RL pricing generate stronger price responses than fixed-pricing policies, and HRL produces the strongest response among them. Higher prices are consistent with margin protection and slower inventory depletion when inventory is scarce. The bottom panel shows that HRL also retains more inventory after disruption-induced drawdowns. Because OUL+Fixed and OUL+RL share the OUL replenishment rule, their difference reflects how pricing changes realized sales and inventory depletion under a common replenishment rule. HRL differs because replenishment and pricing are learned jointly, changing both sales incentives and future stock availability. Figure~\ref{fig:exp1-endpoint-mech} aligns these patterns with the event windows \(S_0,S_1,R_0,R_1\).

\begin{figure}[htbp!]
\centering\includegraphics[width=0.99\linewidth]{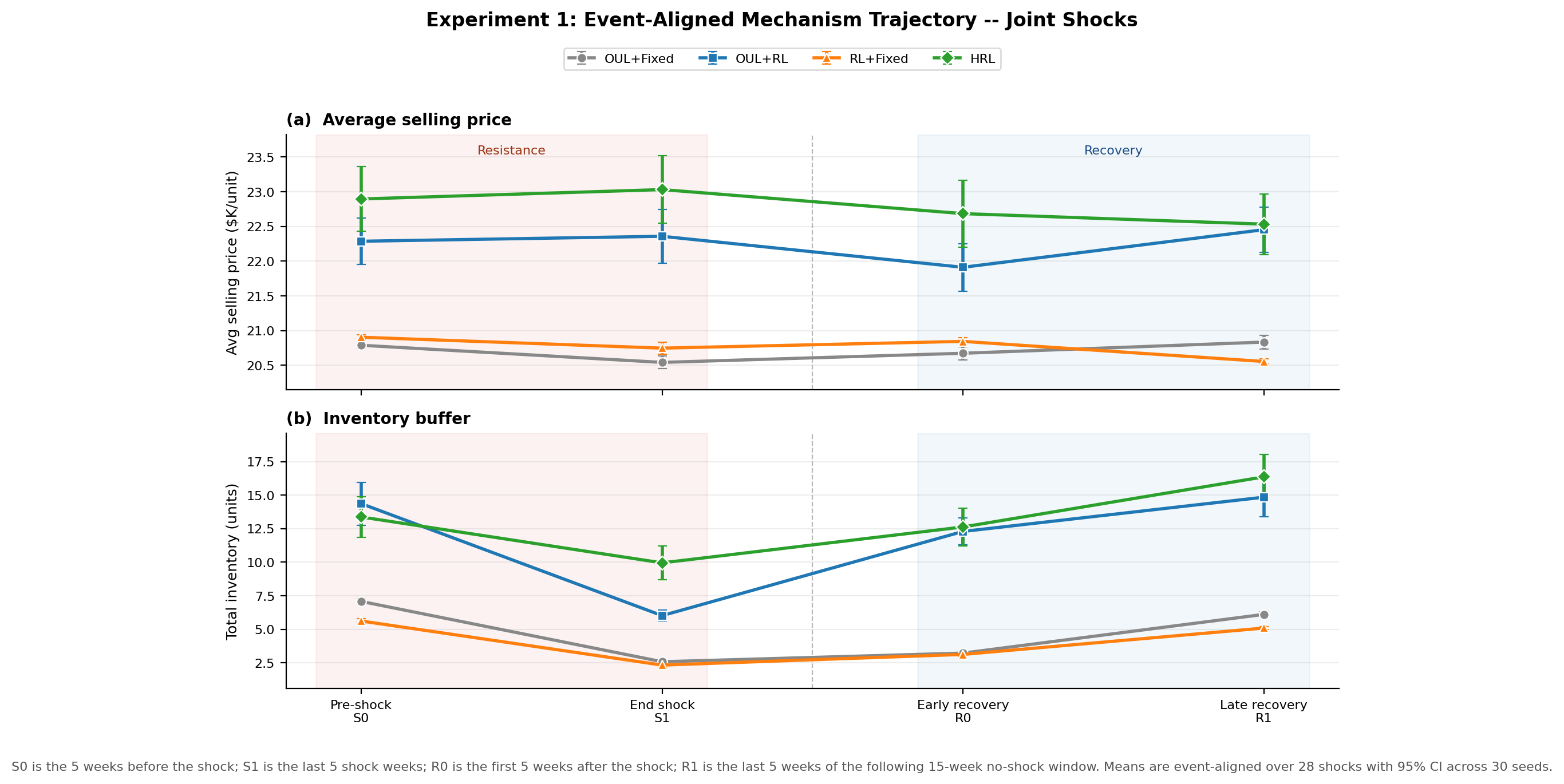}
  \caption{Price and inventory at the four disruption-response endpoints under joint shocks. The endpoints match Figure~\ref{fig:exp1-resilience}. \(S_0\) is the five-period pre-shock endpoint, \(S_1\) is the five-period end-shock endpoint, \(R_0\) is the five-period early-recovery endpoint, and \(R_1\) is the five-period late-recovery endpoint. Panel~(a) reports average selling price, and panel~(b) reports total inventory across 30 seeds.}
  \label{fig:exp1-endpoint-mech}
\end{figure}

Figure~\ref{fig:exp1-endpoint-mech} compares price and inventory at the four disruption-response endpoints. Panel~(a) shows that HRL also has the highest average selling price at the endpoints. It is highest before and at the end of the shock. OUL+RL follows, and the fixed-pricing policies remain below both RL-pricing policies. Panel~(b) clarifies the inventory component of the recovery result. OUL+RL enters the shock with slightly more inventory than HRL, but HRL retains more inventory by \(S_1\). HRL also has the highest inventory point estimate at both recovery endpoints. Thus, the recovery advantage in Figure~\ref{fig:exp1-resilience} is consistent with stronger inventory preservation at the end of the shock and a larger inventory base during recovery. Figure~\ref{fig:exp1-lost-sales} then checks whether the higher-price and inventory-retention pattern is accompanied by more lost sales.

\begin{figure}[htbp!]
  \centering
  \includegraphics[width=0.9\linewidth]{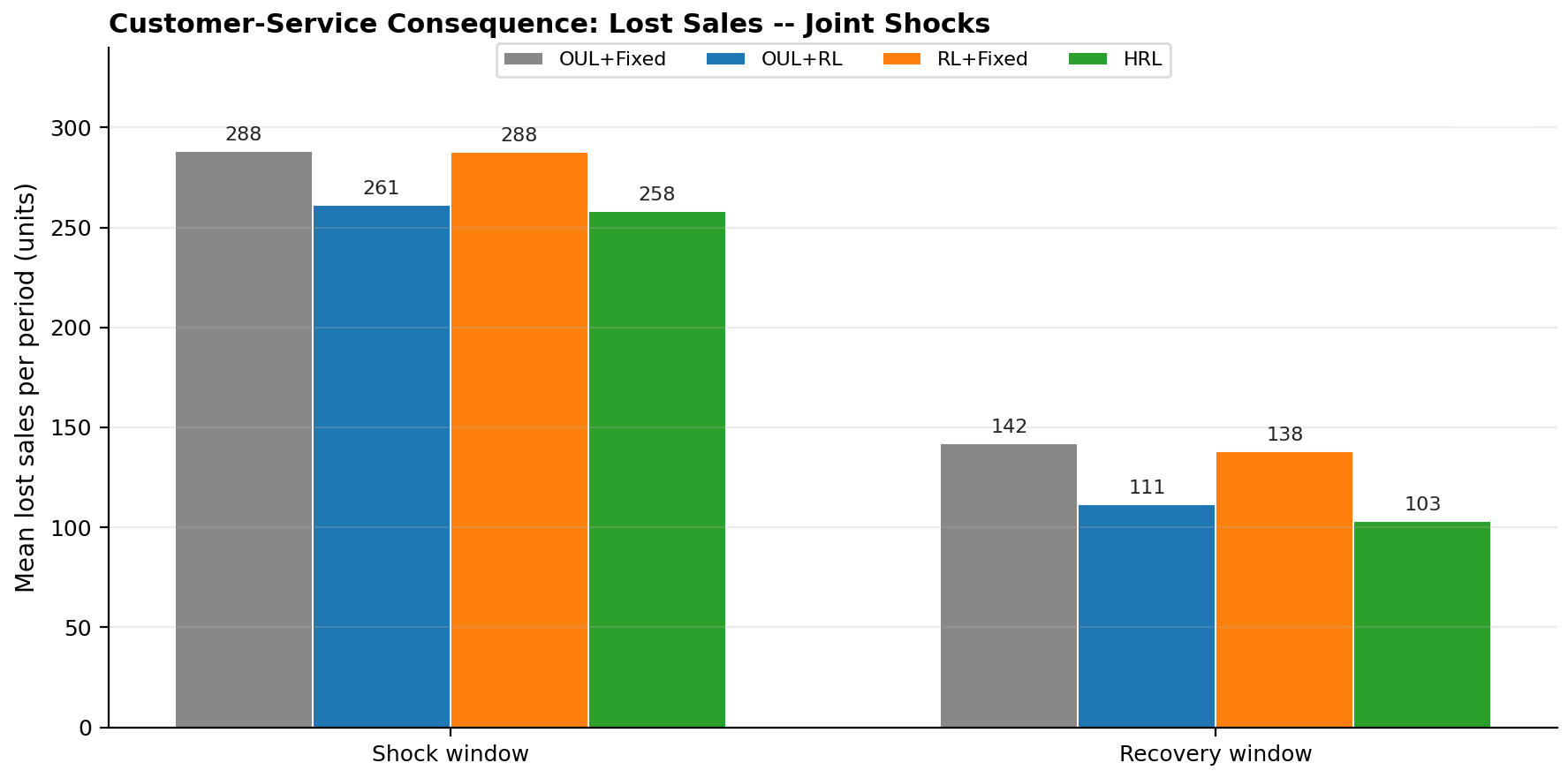}
  \caption{Lost sales during shock and recovery periods. Lower values indicate fewer unserved customers. Bars report event-aligned means for shocks beginning at or after period~1{,}500 across 30 seeds, with values reported above each bar.}
  \label{fig:exp1-lost-sales}
\end{figure}

Figure~\ref{fig:exp1-lost-sales} reports lost sales during shock and recovery periods under joint shocks. All policies have more lost sales during shock periods than during recovery, reflecting more limited inventory availability during disruptions. During shock periods, HRL records the fewest lost sales, with 258 units per period, compared with 261 for OUL+RL and 288 for the fixed-pricing policies. The gap is larger during recovery. HRL records 103 lost sales per period, compared with 111 for OUL+RL and 138--142 for the fixed-pricing policies. Thus, HRL pairs higher prices and larger inventories with fewer lost sales rather than achieving these outcomes by leaving more customers unserved.

Together, Figures~\ref{fig:exp1-price-inv}--\ref{fig:exp1-lost-sales} indicate that HRL's advantage under joint-shock does not come from “more inventory” alone or  “higher prices” alone. Instead, HRL combines higher selling prices, stronger inventory retention, and fewer lost sales during recovery. For dealers, this means the joint-shock advantage is not a single-metric effect, and it should be evaluated across price, inventory, and lost sales, since improving one outcome can weaken recovery if it comes at the expense of another. The observed pattern is consistent with a balanced response: pricing protects margins and moderates inventory depletion during scarcity, while replenishment preserves the inventory base needed to serve future demand. The next experiment examines whether this joint adjustment remains effective during a prolonged disruption.

\subsection{Experiment~2. Prolonged Joint Demand--Supply Shock}
\label{sec:exp2}

Experiment~1 shows that HRL performs most strongly when demand and supply shocks interact. Recent disruptions also show that shocks can persist long enough to deplete operating buffers, as illustrated by the COVID-19 pandemic and disruptions around the Strait of Hormuz \citep{ivanov2020predicting,unctad2026hormuz}. Experiment~2 therefore, studies a prolonged joint demand--supply shock. Two questions guide the analysis. First, \emph{what happens when a joint shock lasts long enough to deplete inventory buffers?} Second, \emph{does coordination create value beyond the separate contributions of learned pricing and learned replenishment?} We address these questions using a \(5{,}500\)-period horizon and introduce a prolonged shock during periods \(3{,}000\)--\(3{,}300\). This shock combines alternating demand surges and drops with simultaneous supply restrictions and class-specific supply caps. The resulting environment creates extended scarcity and tests whether each policy can price and replenish successfully over a longer disruption.

\subsubsection{Profit under the Prolonged Shock}

Table~\ref{tab:exp2-profit} reports mean profit over the full horizon and during the prolonged shock. The shock lowers profit for all four policies, reflecting the difficulty of simultaneous demand and supply disruptions. HRL earns the highest mean profit in both windows. During the shock, it earns \$210K per period, \$22K more than OUL+RL.

\begin{table}[htbp]
  \centering
  \caption{Experiment~2 under a prolonged joint demand--supply shock. Profit summary based on 30 seeds. The evaluation horizon covers periods 450--5{,}500, and the prolonged shock covers periods 3{,}000--3{,}300.}
  \label{tab:exp2-profit}
  \small
  \begin{tabular}{l r@{\;\,$\pm$\;\,}l r@{\;\,$\pm$\;\,}l}
    \toprule
    & \multicolumn{2}{c}{Full horizon (\$K per period)}
    & \multicolumn{2}{c}{Prolonged shock (\$K per period)} \\
    \midrule
    OUL+Fixed  & 182.3 & 0.1  & 152.1 & 0.5 \\
    OUL+RL     & 230.7 & 5.9  & 187.9 & 6.7 \\
    RL+Fixed   & 186.1 & 0.2  & 160.1 & 0.5 \\
    HRL        & \textbf{242.8} & 7.7 & \textbf{210.1} & 7.0 \\
    \bottomrule
  \end{tabular}
\end{table}

The benchmark comparisons also preserve a pattern from Experiment~1. During the prolonged shock, OUL+RL earns \$35.8K more than OUL+Fixed, whereas RL+Fixed earns only \$8.0K more. Pricing therefore remains the primary adaptation lever, but HRL's additional gain over OUL+RL shows that pricing alone does not explain its advantage.

HRL's margin over OUL+RL rises from \$12K per period over the full horizon to \$22K during the shock. Regular periods dilute the full-horizon difference, while the shock window isolates periods when adaptation matters most. OUL+Fixed, with no RL policy at either layer, earns the lowest mean profit in both windows. The larger shock-window margin is consistent with the cumulative role of replenishment. During a prolonged disruption, replenishment shapes inventory across many later arrivals, while pricing adapts within that inventory but cannot change future availability by itself. For dealers, this suggests that adaptive pricing may provide much of the immediate response, but coordinating pricing with replenishment becomes increasingly valuable as a joint disruption persists and inventory buffers erode.

Moreover, Figure~\ref{fig:exp2-weekly} reinforces the table-level result through the corresponding 52-period rolling profit paths. All four policies decline during periods $3,000-3,300$ as the prolonged joint shock unfolds. However, HRL remains above the other policies for much of the shock and returns to a higher profit level afterward. Thus, its advantage is sustained across the disruption and recovery rather than driven by a few high-profit periods. We next use \(2\times2\) contrasts to identify when HRL's joint value emerges across disruption and recovery phases.

\begin{figure}[htbp]
  \centering
  \includegraphics[width=0.99\linewidth]{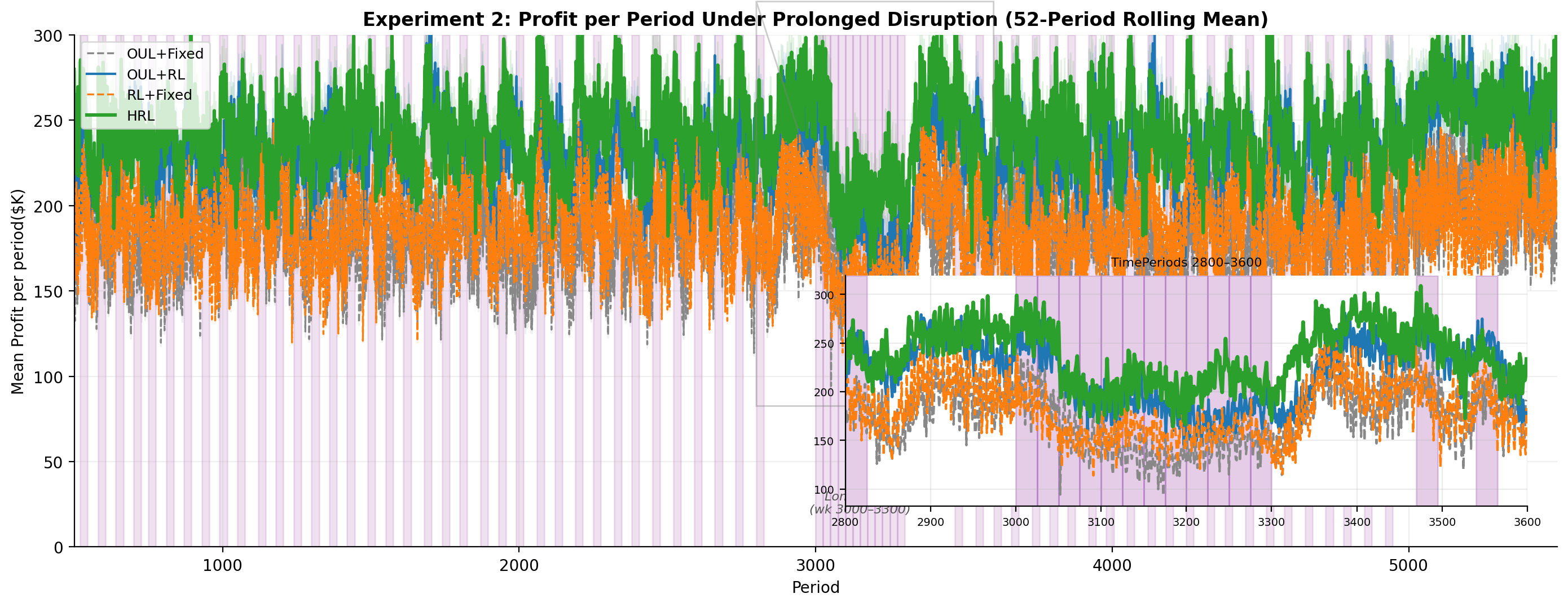}
  \caption{Mean profit per period over the \(5{,}500\)-period horizon with a 52-period rolling mean. The prolonged shock occurs during periods \(3{,}000\)--\(3{,}300\). Shaded bands show 95\% confidence intervals across 30 seeds.}
  \label{fig:exp2-weekly}
\end{figure}

\subsubsection{Two-by-Two Policy Contrasts}

We use the paired \(2\times2\) design to separate the profit gains from adapting replenishment, pricing, and both layers together. OUL+Fixed has no RL policy, OUL+RL adapts pricing, RL+Fixed adapts replenishment, and HRL adapts both. The replenishment and pricing contrasts measure the average gain from adapting each layer, averaging over the other layer. The interaction measures the additional profit from adapting both layers beyond the sum of their standalone gains. It is \(P(\mathrm{HRL})-P(\mathrm{OUL{+}RL})-P(\mathrm{RL{+}Fixed})+P(\mathrm{OUL{+}Fixed})\), where \(P(\cdot)\) denotes mean profit over the relevant period. A positive interaction indicates an additional joint gain, and a confidence interval excluding zero provides clear evidence of that gain.

Table~\ref{tab:exp2-contrasts} reports the three contrasts across evaluation periods. Figure~\ref{fig:exp2-interaction} complements the table with raw mean profits in Panel~(A) and \(2\times2\) policy profiles in Panels~(B)--(D). Parallel profiles indicate additive gains, while nonparallel profiles indicate an interaction.

\begin{table}[htbp]
  \centering
  \caption{Paired \(2\times2\) profit contrasts by evaluation period. Columns report the average gains from RL replenishment, the average gains from RL pricing, and the additional gain from adapting both layers. Values are in \$K per period with 95\% confidence intervals based on 30 paired seeds. Bold estimates have confidence intervals excluding zero.}
  \label{tab:exp2-contrasts}
  \small
  \setlength{\tabcolsep}{5pt}
  \begin{tabular}{@{}lccc@{}}
    \toprule
    Evaluation period
      & \shortstack{Average gain from\\RL replenishment}
      & \shortstack{Average gain from\\RL pricing}
      & \shortstack{Additional gain from\\joint adaptation} \\
    \midrule
    Pre-shock regular
      & \(8.2\) [\(-2.5\), \(19.0\)]
      & \(\mathbf{52.3}\) [\(43.7\), \(61.0\)]
      & \(6.4\) [\(-15.2\), \(27.9\)] \\
    Prolonged shock
      & \(\mathbf{15.1}\) [\(5.4\), \(24.8\)]
      & \(\mathbf{42.9}\) [\(34.4\), \(51.4\)]
      & \(14.2\) [\(-5.2\), \(33.6\)] \\
    \quad Demand-surge periods
      & \(\mathbf{26.3}\) [\(15.1\), \(37.6\)]
      & \(\mathbf{58.1}\) [\(48.5\), \(67.8\)]
      & \(10.7\) [\(-11.6\), \(33.0\)] \\
    \quad Demand-drop periods
      & \(3.9\) [\(-4.5\), \(12.2\)]
      & \(\mathbf{27.6}\) [\(20.0\), \(35.3\)]
      & \(\mathbf{17.7}\) [\(0.6\), \(34.8\)] \\
    \addlinespace[2pt]
    Recovery
      & \(\mathbf{30.5}\) [\(18.8\), \(42.3\)]
      & \(\mathbf{38.7}\) [\(28.5\), \(48.8\)]
      & \(\mathbf{53.4}\) [\(30.4\), \(76.4\)] \\
    \addlinespace[2pt]
    Post-shock regular
      & \(6.7\) [\(-3.7\), \(17.2\)]
      & \(\mathbf{54.9}\) [\(46.0\), \(63.9\)]
      & \(9.9\) [\(-10.9\), \(30.8\)] \\
    \bottomrule
  \end{tabular}
\end{table}

\begin{figure}[t]
  \centering
  \includegraphics[width=0.99\linewidth]{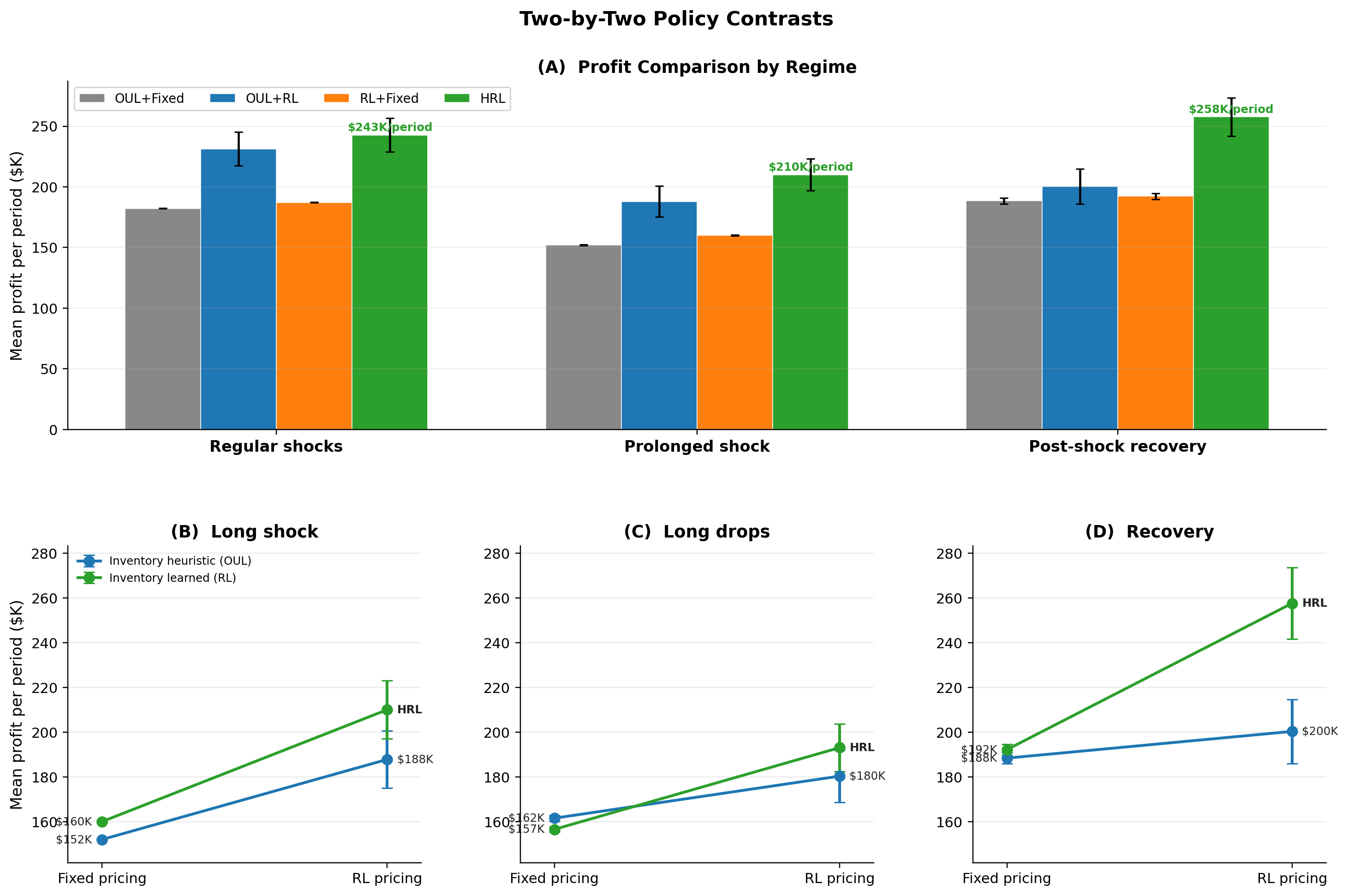}
  \caption{Experiment~2 under a prolonged joint shock. Paired \(2\times2\) policy profiles. Panel~(A) reports mean profit for the periods in Table~\ref{tab:exp2-contrasts}. Panels~(B)--(D) show the policy profiles for the prolonged shock, demand-drop, and recovery periods. Nonparallel lines indicate a nonzero interaction contrast.}
  \label{fig:exp2-interaction}
\end{figure}

First, RL pricing produces a clear profit gain in every evaluation period. The gain ranges from \$27.6K per period during demand drops to \$58.1K during demand surges, and every confidence interval excludes zero. Thus, short-term adaptation remains the most consistent source of profit improvement.

Second, RL replenishment creates clear gains during the prolonged shock, demand surges, and recovery, but not during regular or demand-drop periods. Its gain is largest during recovery at \$30.5K per period. This pattern reflects the cumulative role of replenishment, which shapes the inventory available across future customer arrivals.

Most importantly, the additional gain from adapting both layers becomes clear during demand drops and recovery. The interaction estimates for the full shock and demand surges are positive, but their confidence intervals include zero. In contrast, the interaction reaches \$17.7K per period during demand drops and \$53.4K during recovery, with both confidence intervals excluding zero. During recovery, replenishment changes the inventory available to future customers while pricing determines how that inventory is sold. The positive interaction indicates that adjusting the two decisions jointly creates more profit than the sum of adapting each layer separately.

For dealers, adaptive pricing is the primary tool in regular operations and one-sided disruptions; adaptive replenishment becomes important when shocks persist or recovery begins. HRL is most valuable when pricing protects near-term profit while replenishment restores future inventory.


\section{Discussion and Conclusion}
\label{sec:conclusion}

This paper studies resilient operations as an adaptive two-timescale policy-learning problem. We develop HRL to adapt long-term and short-term policies jointly as shocks change operating conditions. The policies act at their respective decision scales while working toward a shared objective.

Our used-car case study demonstrates how adaptation at both decision levels improves resilience. Short-term adaptation responds immediately to changing market conditions, whereas long-term adaptation reshapes the operating environment faced by future short-term decisions. Together, the two policies improve both operational performance and recovery from disruptions.

Four managerial insights emerge. First, the proposed HRL framework is modular and can be implemented incrementally or partially, allowing organizations to preserve existing planning and decision-making processes. This feature lowers implementation barriers and facilitates adoption across organizational units, where changes to established business processes are often difficult to achieve. Second, in stable operating environments, the HRL framework performs comparably to existing policies within statistical confidence bounds. This property is essential for practical adoption because it allows firms to improve resilience without sacrificing performance during normal operations. Third, the benefits of joint adaptation become substantial when demand and supply are disrupted simultaneously. Dynamic pricing alone is often sufficient to address routine fluctuations or one-sided shocks but coordinated pricing and inventory adaptation become increasingly valuable when disruptions affect both demand and supply. The greatest advantage of HRL occurs during the recovery phase, when replenishment decisions rebuild the inventory mix while pricing protects margins and guides demand. Managers should therefore view HRL not only as a disruption-response policy but also as a recovery policy that coordinates pricing, inventory availability, and demand fulfillment.

Finally, the operational benefits are economically meaningful. Compared with the strongest partially adaptive benchmark, the proposed framework increases mean profit by 9.2\% under joint demand–supply shocks and by 11.8\% under prolonged disruptions while simultaneously reducing profit variability. These results demonstrate that coordinated adaptation across hierarchical decision levels improves both profitability and resilience.
Several opportunities for future research remain. Extending the framework to additional hierarchical decision levels, multiple interacting products, competitive environments, and richer supply-chain settings would further broaden its applicability. Another promising direction is incorporating model-based forecasts or causal information into the learning process to accelerate adaptation under rapidly changing operating conditions.
}

\clearpage
  \bibliographystyle{plainnat}
  \bibliography{ref}

\beginsupplementary
\thispagestyle{firstpagestyle}

\begin{center}
{\fontsize{21}{25}\selectfont \supplementaryname\par}
\vspace{0.65em}
{\fontsize{15}{18}\selectfont
``Two-Timescale Hierarchical Reinforcement Learning
for Resilient Operations''\par}
\end{center}

\vspace{1.1em}

\IfFileExists{supplementary.tex}{%
  \input{supplementary}
}{%
This Supplementary provides proofs and implementation details supporting the main paper. Section~\ref{app:proof_theorem1} proves Theorem~\ref{thm:global_convergence_generic}. Section~\ref{app:proof_accelerated_convergence} proves Theorem~\ref{thm:rate_improve_average}, including the accelerated average-gap bound under market sharpness. Section~\ref{app:case_study_details} describes the simulation environment, demand and inventory dynamics, policy implementations, training protocol, and statistical reporting for the used-car case study. Section~\ref{sec:additional_numerical} presents controlled numerical illustrations of the mechanisms underlying Theorems~\ref{thm:global_convergence_generic} and~\ref{thm:rate_improve_average}.

\section{Proof of Theorem~\ref{thm:global_convergence_generic}}
\label{app:proof_theorem1}

\begin{proof}[Proof of Theorem~\ref{thm:global_convergence_generic}]
We prove the result for the population-level policy-improvement sequence described in the main text. The empirical algorithm uses parameterized PPO-Clip with sampled advantages. The argument below abstracts from sampling, function approximation, and finite inner-optimization errors, and studies exact advantage directions. 

Throughout the proof, Assumptions~1--3 hold. The long-term state space, short-term state space, and feasible action sets are finite. Rewards and costs are bounded by \(B\), all policies used in the learning sequence assign positive probability to feasible actions, and exact long-term and short-term advantages are used. The short-term horizon within one long-term period is \(K\), while \(M\ge1\) denotes the number of short-term policy updates performed within one long-term period. The quantities \(K\) and \(M\) are fixed over the learning horizon \(T\).

At a decision node, the population update takes the form
\[
p^+(a)
=
\frac{p(a)\exp\{\eta A(a)\}}
{\sum_b p(b)\exp\{\eta A(b)\}} .
\]
Let \(Z_p(\eta)=\sum_b p(b)\exp\{\eta A(b)\}\). For any comparator \(q\),
\[
\begin{aligned}
D_{\mathrm{KL}}(q\|p^+)-D_{\mathrm{KL}}(q\|p)=
\sum_a q(a)\log\frac{p(a)}{p^+(a)} = -\eta\sum_a q(a)A(a)+\log Z_p(\eta).
\end{aligned}
\]
If \(\|A\|_\infty\le G\), Hoeffding's lemma gives \[ \log Z_p(\eta) = \log \E_{a\sim p}\{\exp(\eta A(a))\} \le \eta\sum_a p(a)A(a)+\frac{\eta^2G^2}{2}. \] Hence, \[ D_{\mathrm{KL}}(q\|p^+)-D_{\mathrm{KL}}(q\|p) \le -\eta\langle A,q-p\rangle + \frac{\eta^2G^2}{2}. \]
The update also satisfies a movement bound. Since $\frac{p^+(a)}{p(a)}=\frac{\exp\{\eta A(a)\}}{Z_p(\eta)}$ and \(e^{-\eta G}\le Z_p(\eta)\le e^{\eta G}\), we have \(e^{-2\eta G}\le p^+(a)/p(a)\le e^{2\eta G}\). Hence, for any stepsize range \(\eta\le\bar\eta\),
\[
\|p^+-p\|_1
=
\sum_a p(a)\left|\frac{p^+(a)}{p(a)}-1\right|
\le
e^{2\eta G}-1
\le
2Ge^{2\bar\eta G}\eta.
\]
In the short-term-policy bounds below, \(B_{\mathrm f}\) denotes the corresponding finite movement constant, independent of \(T\).

\noindent
\textbf{Step 1. Uniform value and advantage bounds.} The period payoff is \(R_t=\sum_{k=0}^{K-1}r_{t,k}-c(x_t,u_t)\). Since \(|r_{t,k}|\le B\) and \(|c(x_t,u_t)|\le B\), we have \(|R_t|\le (K+1)B\). Write \(R_B=(K+1)B\). For every stationary joint policy \((\pi,\phi)\), \[ \|V_{\mathrm s}^{\pi,\phi}\|_\infty \le \frac{R_B}{1-\Gamma}, \qquad \|A_{\mathrm s}^{\pi,\phi}\|_\infty \le \frac{2R_B}{1-\Gamma}. \] The long-term update uses the scaled long-term direction \(g_t^{\mathrm s}=A_{\mathrm s}^{\pi_t,\phi_{t+1}}/(1-\Gamma)\), so \[ \|g_t^{\mathrm s}\|_\infty \le G_{\mathrm s} \triangleq \frac{2R_B}{(1-\Gamma)^2}. \] For the short-term problem, a short-term advantage includes remaining within-period rewards and the terminal continuation value \(\Gamma V_{\mathrm s}^{\pi_t,\phi}(x_{t+1})\). Since the short-term horizon is \(K\), \[ \|A_{\mathrm f}^{\pi_t,\phi}\|_\infty \le G_{\mathrm f} \triangleq 2\left(KB+\frac{\Gamma R_B}{1-\Gamma}\right) \] for every short-term policy \(\phi\) used in the proof. Write \(\phi_{t,0}=\phi_t\), \(\phi_{t,M}=\phi_{t+1}\), and \(g_{t,m}^{\mathrm f}=A_{\mathrm f}^{\pi_t,\phi_{t,m}}\).

\noindent
\textbf{Step 2. Variation of the short-term advantage.}
We next derive the finite-horizon variation bound used later in the proof. Fix a long-term policy \(\pi\) and two short-term policies \(\phi_1\) and \(\phi_2\). Let
\[
\delta_\phi
=
\|\phi_1-\phi_2\|_{1,\infty}
=
\sup_{s,u,k}
\sum_a
|\phi_1(a\mid s,u,k)-\phi_2(a\mid s,u,k)|.
\]
For \(i=1,2\), let \(\bar r_i\) and \(\bar P_i\) denote the one-period expected payoff vector and long-term transition matrix induced by \((\pi,\phi_i)\). Since \(K\) is finite and the state and action spaces are finite, \(\bar r_i\) and \(\bar P_i\) are finite sums over within-period trajectories of length \(K\). For a fixed within-period action sequence, the short-term action probabilities satisfy
\[
\left|
\prod_{k=0}^{K-1}\phi_1(a_k\mid s_k,u,k)
-
\prod_{k=0}^{K-1}\phi_2(a_k\mid s_k,u,k)
\right|
\le
\sum_{\ell=0}^{K-1}
|\phi_1(a_\ell\mid s_\ell,u,\ell)-\phi_2(a_\ell\mid s_\ell,u,\ell)|.
\]
All remaining transition probabilities are bounded by one. Since the number of within-period states and actions is finite, there exist finite constants \(C_{r,K}\) and \(C_{P,K}\), independent of \(T\), such that
\[
\|\bar r_1-\bar r_2\|_\infty
\le
C_{r,K}\delta_\phi,
\qquad
\|\bar P_1-\bar P_2\|_\infty
\le
C_{P,K}\delta_\phi.
\]

Let \(V_i=V_{\mathrm s}^{\pi,\phi_i}\). The long-term value equations are
\[
V_i=\bar r_i+\Gamma \bar P_iV_i,
\qquad i=1,2.
\]
Subtracting the two equations gives
\[
V_1-V_2
=
(I-\Gamma\bar P_1)^{-1}
\{(\bar r_1-\bar r_2)+\Gamma(\bar P_1-\bar P_2)V_2\}.
\]
Because \(\|(I-\Gamma\bar P_1)^{-1}\|_\infty\le(1-\Gamma)^{-1}\) and \(\|V_2\|_\infty\le R_B/(1-\Gamma)\),
\[
\begin{aligned}
\|V_1-V_2\|_\infty
\le
\frac{1}{1-\Gamma}
\left\{
\|\bar r_1-\bar r_2\|_\infty
+
\Gamma\|\bar P_1-\bar P_2\|_\infty\|V_2\|_\infty
\right\}  \le
\frac{1}{1-\Gamma}
\left\{
C_{r,K}
+
\frac{\Gamma R_B}{1-\Gamma}C_{P,K}
\right\}
\delta_\phi.
\end{aligned}
\]
Thus there is a finite constant \(L_V<\infty\), independent of \(T\), such that
\[
\|V_{\mathrm s}^{\pi,\phi_1}-V_{\mathrm s}^{\pi,\phi_2}\|_\infty
\le
L_V\delta_\phi.
\]

We now pass this bound to the short-term advantage. For \(i=1,2\), let \(W_{i,k}(s,u)\) be the short-term continuation value at position \(k\), and let \(Q_{i,k}(s,u,a)\) be the corresponding short-term action value under \((\pi,\phi_i)\). The terminal value is
\[
W_{i,K}(s,u)
=
\Gamma\sum_{x'}P_{\mathrm s}(x'\mid s,u)V_i(x').
\]
For \(k=0,\ldots,K-1\),
\[
Q_{i,k}(s,u,a)
=
r_k(s,a,u)
+
\sum_{s'}P_{\mathrm f,k}(s'\mid s,a,u)W_{i,k+1}(s',u),
\]
and
\[
W_{i,k}(s,u)
=
\sum_a
\phi_i(a\mid s,u,k)Q_{i,k}(s,u,a).
\]
Let \(Q_B=KB+\Gamma R_B/(1-\Gamma)\). Then \(\|Q_{i,k}\|_\infty\le Q_B\). From the terminal condition,
\[
\|W_{1,K}-W_{2,K}\|_\infty
\le
\Gamma\|V_1-V_2\|_\infty
\le
\Gamma L_V\delta_\phi.
\]
For \(k=K-1,\ldots,0\),
\[
\|Q_{1,k}-Q_{2,k}\|_\infty
\le
\|W_{1,k+1}-W_{2,k+1}\|_\infty,
\]
and
\[
\begin{aligned}
\|W_{1,k}-W_{2,k}\|_\infty
&\le
\sup_{s,u}
\left|
\sum_a
\{\phi_1(a\mid s,u,k)-\phi_2(a\mid s,u,k)\}
Q_{1,k}(s,u,a)
\right|  \\
&\quad+
\sup_{s,u}
\left|
\sum_a
\phi_2(a\mid s,u,k)
\{Q_{1,k}(s,u,a)-Q_{2,k}(s,u,a)\}
\right| \\
&\le
Q_B\delta_\phi
+
\|Q_{1,k}-Q_{2,k}\|_\infty.
\end{aligned}
\]
Backward induction over the finite horizon \(K\) gives a finite constant \(C_{Q,K}<\infty\), independent of \(T\), such that
\[
\max_{0\le k\le K}
\|W_{1,k}-W_{2,k}\|_\infty
+
\max_{0\le k\le K-1}
\|Q_{1,k}-Q_{2,k}\|_\infty
\le
C_{Q,K}\delta_\phi.
\]
Since \(A_{i,k}(s,u,a)=Q_{i,k}(s,u,a)-W_{i,k}(s,u)\),
\[
\|A_{\mathrm f}^{\pi,\phi_1}-A_{\mathrm f}^{\pi,\phi_2}\|_\infty
\le
2C_{Q,K}\delta_\phi.
\]
Writing \(L_{A,K}=2C_{Q,K}\), we obtain
\[
\|A_{\mathrm f}^{\pi,\phi_1}-A_{\mathrm f}^{\pi,\phi_2}\|_\infty
\le
L_{A,K}
\|\phi_1-\phi_2\|_{1,\infty}.
\]
The constant \(L_{A,K}\) may depend on \(K\), \(B\), \(\Gamma\), the finite state and action spaces, and the transition probabilities, but not on \(T\).

\noindent
\textbf{Step 3. Joint performance-difference decomposition.}
Fix the learning iteration \(t\). Let \(\Pr^*\) and \(\E^*\) denote probability and expectation under the benchmark pair \((\pi^*,\phi^*)\). Let \(V_{\mathrm s}^t=V_{\mathrm s}^{\pi_t,\phi_{t+1}}\). Since \(\mathcal L(\pi_t,\phi_{t+1})=\E_{\rho_0}\{V_{\mathrm s}^t(x_0)\}\),
\[
\Gap_t
=
\E^*
\left[
\sum_{\tau\ge0}\Gamma^\tau R_\tau
-
V_{\mathrm s}^t(x_0)
\right].
\]
The discounted telescoping identity gives
\[
\begin{aligned}
\E^*
\left[
\sum_{\tau\ge0}\Gamma^\tau R_\tau
-
V_{\mathrm s}^t(x_0)
\right]
&=
\E^*
\left[
\sum_{\tau\ge0}\Gamma^\tau
\{R_\tau+\Gamma V_{\mathrm s}^t(x_{\tau+1})-V_{\mathrm s}^t(x_\tau)\}
\right].
\end{aligned}
\]
The terminal term vanishes because \(\Gamma<1\) and \(V_{\mathrm s}^t\) is bounded.

Let \(A_{\mathrm s}^t\) and \(A_{\mathrm f}^t\) denote the long-term and short-term advantages evaluated under the post-short-term-update pair \((\pi_t,\phi_{t+1})\). Let \(W_k^t(s,u)\) be the short-term continuation value at position \(k\) under this pair, with terminal condition \[ W_K^t(s,u) = \Gamma\sum_{x'}P_{\mathrm s}(x'\mid s,u)V_{\mathrm s}^t(x'). \] Let \(Q_{\mathrm f,k}^t(s,u,a)=r_k(s,a,u)+\sum_{s'}P_{\mathrm f,k}(s'\mid s,a,u)W_{k+1}^t(s',u)\), so that \(A_{\mathrm f}^t(s,u,k,a)=Q_{\mathrm f,k}^t(s,u,a)-W_k^t(s,u)\). Also define \[ Q_{\mathrm s}^t(x,u) = -c(x,u) + \sum_{s_0}P_{\mathrm{init}}(s_0\mid x,u)W_0^t(s_0,u), \qquad A_{\mathrm s}^t(x,u)=Q_{\mathrm s}^t(x,u)-V_{\mathrm s}^t(x). \] short-term-level Bellman telescoping gives, after averaging over the initialized short-term state and the short-term transitions conditional on \((x_\tau,u_\tau)\), \[ \E^* \left[ R_\tau+\Gamma V_{\mathrm s}^t(x_{\tau+1})-V_{\mathrm s}^t(x_\tau) \mid x_\tau,u_\tau \right] = A_{\mathrm s}^{t}(x_\tau,u_\tau) + \E^* \left[ \sum_{k=0}^{K-1} A_{\mathrm f}^{t}(s_{\tau,k},u_\tau,k,a_{\tau,k}) \mid x_\tau,u_\tau \right]. \] 
Therefore, 
\[ 
\Gap_t = \E^* \left[ \sum_{\tau\ge0}\Gamma^\tau \left\{ A_{\mathrm s}^{t}(x_\tau,u_\tau) + \sum_{k=0}^{K-1} A_{\mathrm f}^{t}(s_{\tau,k},u_\tau,k,a_{\tau,k}) \right\} \right]. 
\]

Let \(d_{\mathrm s}^*\) be the normalized discounted long-term occupancy under the benchmark,
\[
d_{\mathrm s}^*(x)
=
(1-\Gamma)
\sum_{\tau\ge0}
\Gamma^\tau
\Pr^*(x_\tau=x),
\]
and let \(d_{\mathrm f}^*\) be the benchmark occupancy over short-term decision nodes, averaged over short-term positions,
\[
d_{\mathrm f}^*(s,u,k)
=
\frac{1-\Gamma}{K}
\sum_{\tau\ge0}
\Gamma^\tau
\Pr^*(s_{\tau,k}=s,u_\tau=u).
\]
For any long-term action distribution \(\alpha\) and any long-term function \(h\), define
\[
\langle h,\alpha-\pi_t\rangle_{\mathrm s}
=
\sum_x d_{\mathrm s}^*(x)
\sum_u
\{\alpha(u\mid x)-\pi_t(u\mid x)\}h(x,u).
\]
For any short-term action distribution \(\alpha\) and any short-term function \(H\), define
\[
\langle H,\alpha-\phi\rangle_{\mathrm f}
=
\sum_{s,u,k}d_{\mathrm f}^*(s,u,k)
\sum_a
\{\alpha(a\mid s,u,k)-\phi(a\mid s,u,k)\}H(s,u,k,a).
\]
At each long-term state, \(\sum_u\pi_t(u\mid x)A_{\mathrm s}^t(x,u)=0\). At each short-term decision node, \(\sum_a\phi_{t+1}(a\mid s,u,k)A_{\mathrm f}^t(s,u,k,a)=0\). Hence
\[
\E_{u\sim\pi^*(\cdot\mid x)}
A_{\mathrm s}^t(x,u)
=
\sum_u
\{\pi^*(u\mid x)-\pi_t(u\mid x)\}
A_{\mathrm s}^t(x,u),
\]
and
\[
\E_{a\sim\phi^*(\cdot\mid s,u,k)}
A_{\mathrm f}^t(s,u,k,a)
=
\sum_a
\{\phi^*(a\mid s,u,k)-\phi_{t+1}(a\mid s,u,k)\}
A_{\mathrm f}^t(s,u,k,a).
\]
Using the definitions of \(d_{\mathrm s}^*\) and \(d_{\mathrm f}^*\),
\[
\Gap_t
=
\frac{1}{1-\Gamma}
\left\langle A_{\mathrm s}^{t},\pi^*-\pi_t\right\rangle_{\mathrm s}
+
\frac{K}{1-\Gamma}
\left\langle A_{\mathrm f}^{t},\phi^*-\phi_{t+1}\right\rangle_{\mathrm f}.
\]
Using \(g_t^{\mathrm s}=A_{\mathrm s}^t/(1-\Gamma)\), \(g_{t,M}^{\mathrm f}=A_{\mathrm f}^t\), and \(\phi_{t,M}=\phi_{t+1}\), we obtain
\[
\Gap_t
=
\left\langle g_t^{\mathrm s},\pi^*-\pi_t\right\rangle_{\mathrm s}
+
\frac{K}{1-\Gamma}
\left\langle g_{t,M}^{\mathrm f},\phi^*-\phi_{t,M}\right\rangle_{\mathrm f}.
\]

\noindent
\textbf{Step 4. Long-term and short-term KL drift.} Define \[ \Phi(\pi,\phi) = \WKL{d_{\mathrm s}^*}(\pi^*\|\pi) + \WKL{d_{\mathrm f}^*}(\phi^*\|\phi), \qquad \Phi_t=\Phi(\pi_t,\phi_t). \] The one-period drift is \[ \Phi_{t+1}-\Phi_t = \{\Phi(\pi_{t+1},\phi_{t+1})-\Phi(\pi_t,\phi_{t+1})\} + \{\Phi(\pi_t,\phi_{t+1})-\Phi(\pi_t,\phi_t)\}. \] Applying the one-step inequality to the long-term update and averaging under \(d_{\mathrm s}^*\) gives \[ \Phi(\pi_{t+1},\phi_{t+1})-\Phi(\pi_t,\phi_{t+1}) \le -\eta_{\mathrm s} \left\langle g_t^{\mathrm s},\pi^*-\pi_t\right\rangle_{\mathrm s} + \frac{\eta_{\mathrm s}^2G_{\mathrm s}^2}{2}. \] Using the decomposition from Step 3, 
\[ \begin{aligned} \Phi(\pi_{t+1},\phi_{t+1})-\Phi(\pi_t,\phi_{t+1}) \le -\eta_{\mathrm s}\Gap_t + \eta_{\mathrm s}\frac{K}{1-\Gamma} \left\langle g_{t,M}^{\mathrm f},\phi^*-\phi_{t,M}\right\rangle_{\mathrm f}+ \frac{\eta_{\mathrm s}^2G_{\mathrm s}^2}{2}. \end{aligned} \] 

For the short-term updates, summing the one-step inequality over \(m=0,\ldots,M-1\) gives 

\[ 
\begin{aligned} \Phi(\pi_t,\phi_{t+1})-\Phi(\pi_t,\phi_t) \le -\eta_{\mathrm f} \sum_{m=0}^{M-1} \left\langle g_{t,m}^{\mathrm f},\phi^*-\phi_{t,m}\right\rangle_{\mathrm f} + \frac{M\eta_{\mathrm f}^2G_{\mathrm f}^2}{2}. \end{aligned} 
\] The movement bound gives \(\|\phi_{t,m+1}-\phi_{t,m}\|_{1,\infty}\le B_{\mathrm f}\eta_{\mathrm f}\). Hence, for any \(m\le M\), \[ \|\phi_{t,M}-\phi_{t,m}\|_{1,\infty} \le \sum_{\ell=m}^{M-1} \|\phi_{t,\ell+1}-\phi_{t,\ell}\|_{1,\infty} \le (M-m)B_{\mathrm f}\eta_{\mathrm f}. \]

\noindent
\textbf{Step 5. Synchronization and residual control.}
Impose the synchronization condition \(M\eta_{\mathrm f}=\{K/(1-\Gamma)\}\eta_{\mathrm s}\). Combining the long-term and short-term drift bounds from Step 4 gives
\[
\begin{aligned}
\Phi_{t+1}-\Phi_t
\le
-\eta_{\mathrm s}\Gap_t
+
\frac{\eta_{\mathrm s}^2G_{\mathrm s}^2}{2}
+
\frac{M\eta_{\mathrm f}^2G_{\mathrm f}^2}{2}
+\eta_{\mathrm s}\frac{K}{1-\Gamma}
\left\langle g_{t,M}^{\mathrm f},\phi^*-\phi_{t,M}\right\rangle_{\mathrm f}-
\eta_{\mathrm f}
\sum_{m=0}^{M-1}
\left\langle g_{t,m}^{\mathrm f},\phi^*-\phi_{t,m}\right\rangle_{\mathrm f}.
\end{aligned}
\]
Since \(M\eta_{\mathrm f}=\{K/(1-\Gamma)\}\eta_{\mathrm s}\),
\[
\eta_{\mathrm s}\frac{K}{1-\Gamma}
\left\langle g_{t,M}^{\mathrm f},\phi^*-\phi_{t,M}\right\rangle_{\mathrm f}
=
\eta_{\mathrm f}
\sum_{m=0}^{M-1}
\left\langle g_{t,M}^{\mathrm f},\phi^*-\phi_{t,M}\right\rangle_{\mathrm f}.
\]
Thus the remaining short-term first-order term is the residual
\[
\mathcal R_t
=
\eta_{\mathrm f}
\sum_{m=0}^{M-1}
\left[
\left\langle g_{t,M}^{\mathrm f},\phi^*-\phi_{t,M}\right\rangle_{\mathrm f}
-
\left\langle g_{t,m}^{\mathrm f},\phi^*-\phi_{t,m}\right\rangle_{\mathrm f}
\right].
\]
For a fixed \(m\), add and subtract
\(\langle g_{t,M}^{\mathrm f},\phi^*-\phi_{t,m}\rangle_{\mathrm f}\). Then
\[
\begin{aligned}
\left|
\left\langle g_{t,M}^{\mathrm f},\phi^*-\phi_{t,M}\right\rangle_{\mathrm f}
-
\left\langle g_{t,m}^{\mathrm f},\phi^*-\phi_{t,m}\right\rangle_{\mathrm f}
\right| \le
\left|
\left\langle g_{t,M}^{\mathrm f}-g_{t,m}^{\mathrm f},\phi^*-\phi_{t,m}\right\rangle_{\mathrm f}
\right|
+
\left|
\left\langle g_{t,M}^{\mathrm f},\phi_{t,m}-\phi_{t,M}\right\rangle_{\mathrm f}
\right|.
\end{aligned}
\]
By Step 2 and the movement bound,
\[
\begin{aligned}
\|g_{t,M}^{\mathrm f}-g_{t,m}^{\mathrm f}\|_\infty=
\|A_{\mathrm f}^{\pi_t,\phi_{t,M}}-A_{\mathrm f}^{\pi_t,\phi_{t,m}}\|_\infty
\le
L_{A,K}\|\phi_{t,M}-\phi_{t,m}\|_{1,\infty} 
\le
L_{A,K}(M-m)B_{\mathrm f}\eta_{\mathrm f}.
\end{aligned}
\]
Since \(\|\phi^*-\phi_{t,m}\|_1\le2\) at each short-term decision node,
\[
\left|
\left\langle g_{t,M}^{\mathrm f}-g_{t,m}^{\mathrm f},\phi^*-\phi_{t,m}\right\rangle_{\mathrm f}
\right|
\le
2L_{A,K}(M-m)B_{\mathrm f}\eta_{\mathrm f}.
\]
The second term is bounded by \(\|g_{t,M}^{\mathrm f}\|_\infty\le G_{\mathrm f}\) and the movement bound,
\[
\left|
\left\langle g_{t,M}^{\mathrm f},\phi_{t,m}-\phi_{t,M}\right\rangle_{\mathrm f}
\right|
\le
G_{\mathrm f}(M-m)B_{\mathrm f}\eta_{\mathrm f}.
\]
Therefore,
\[
\begin{aligned}
&
\left|
\left\langle g_{t,M}^{\mathrm f},\phi^*-\phi_{t,M}\right\rangle_{\mathrm f}
-
\left\langle g_{t,m}^{\mathrm f},\phi^*-\phi_{t,m}\right\rangle_{\mathrm f}
\right| 
\le
B_{\mathrm f}(2L_{A,K}+G_{\mathrm f})(M-m)\eta_{\mathrm f}.
\end{aligned}
\]
Summing over \(m\),
\[
\begin{aligned}
|\mathcal R_t|
&\le
\eta_{\mathrm f}
\sum_{m=0}^{M-1}
B_{\mathrm f}(2L_{A,K}+G_{\mathrm f})(M-m)\eta_{\mathrm f} 
\le
B_{\mathrm f}(2L_{A,K}+G_{\mathrm f})M^2\eta_{\mathrm f}^2 
=
\left(\frac{K}{1-\Gamma}\right)^2
B_{\mathrm f}(2L_{A,K}+G_{\mathrm f})
\eta_{\mathrm s}^2.
\end{aligned}
\]
The short-term second-order term satisfies
\[
\frac{M\eta_{\mathrm f}^2G_{\mathrm f}^2}{2}
=
\left(\frac{K}{1-\Gamma}\right)^2
\frac{\eta_{\mathrm s}^2G_{\mathrm f}^2}{2M}.
\]
Taking expectation if the update contains external randomness,
\[
\E[\Phi_{t+1}-\Phi_t]
\le
-\eta_{\mathrm s}\E[\Gap_t]
+
C_1\eta_{\mathrm s}^2,
\]
where one valid finite constant is
\[
C_1
=
\frac{G_{\mathrm s}^2}{2}
+
\left(\frac{K}{1-\Gamma}\right)^2
\frac{G_{\mathrm f}^2}{2M}
+
\left(\frac{K}{1-\Gamma}\right)^2
B_{\mathrm f}(2L_{A,K}+G_{\mathrm f}).
\]
The constant \(C_1\) may depend on \(K\), \(M\), \(\Gamma\), \(B\), the finite state and action spaces, the transition probabilities, and \(L_{A,K}\), but not on \(T\). Since the state component \(Z_t\) can affect rewards, feasible actions, and transition probabilities, shocks enter \(C_1\) through these fixed problem quantities.

\noindent
\textbf{Step 6. Telescoping the drift.} Rearranging the drift inequality gives 
\[ 
\eta_{\mathrm s}\E[\Gap_t] \le \E[\Phi_t]-\E[\Phi_{t+1}] + C_{1}\eta_{\mathrm s}^2. 
\] Summing over \(t=0,\ldots,T-1\), 
\[ 
\eta_{\mathrm s} \sum_{t=0}^{T-1} \E[\Gap_t] \le \E[\Phi_0]-\E[\Phi_T] + TC_{1}\eta_{\mathrm s}^2. 
\] 
Since \(\Phi_T\ge0\), 
\[ \eta_{\mathrm s} \sum_{t=0}^{T-1} \E[\Gap_t] \le \Phi_0 + TC_{1}\eta_{\mathrm s}^2. 
\] 

Let \(C_0=\Phi_0\). Dividing by \(T\eta_{\mathrm s}\) gives
\[
\frac{1}{T}
\sum_{t=0}^{T-1}
\E[\Gap_t]
\le
\frac{C_0}{T\eta_{\mathrm s}}
+
C_1\eta_{\mathrm s}.
\]
With \(\eta_{\mathrm s}=T^{-1/2}\),
\[
\frac{1}{T}
\sum_{t=0}^{T-1}
\E[\Gap_t]
\le
\frac{C_0+C_1}{\sqrt T}.
\]

Since \(C_0\) and \(C_1\) are finite and independent of \(T\), the claimed bound follows.
\end{proof}
\section{Proof of Theorem~\ref{thm:rate_improve_average}}
\label{app:proof_accelerated_convergence}

\begin{proof}[Proof of Theorem~\ref{thm:rate_improve_average}]
The proof refines the drift bound from Theorem~\ref{thm:global_convergence_generic} using Assumption~4. Let
\[
\Delta_t=\E[\Phi(\pi_t,\phi_t)],
\qquad
\bar\Gap_t=\E[\Gap_t].
\]
The proof of Theorem~\ref{thm:global_convergence_generic} applies to time-varying learning rates when the synchronization condition is imposed at every long-term period as \(M\eta_{\mathrm f,t}=\{K/(1-\Gamma)\}\eta_{\mathrm s,t}\). Hence there is a finite constant \(C_1\), independent of \(t\) and \(T\), such that
\[
\Delta_{t+1}
\le
\Delta_t-\eta_{\mathrm s,t}\bar\Gap_t+C_1\eta_{\mathrm s,t}^2.
\]

\noindent
\textbf{Step 1. Relate the post-short-term-update gap to Assumption~4.}
Assumption~4 applies to the pair \((\pi_t,\phi_t)\), while \(\Gap_t\) is evaluated at \((\pi_t,\phi_{t+1})\). We first control this difference.

By the finite-horizon variation argument in the proof of Theorem~\ref{thm:global_convergence_generic}, there exists a finite constant \(L_{\mathcal L,K}\), independent of \(T\), such that for any long-term policy \(\pi\) and any two short-term policies \(\phi_1,\phi_2\),
\[
\left|
\mathcal L(\pi,\phi_1)-\mathcal L(\pi,\phi_2)
\right|
\le
L_{\mathcal L,K}
\|\phi_1-\phi_2\|_{1,\infty}.
\]
The short-term movement bound from Theorem~\ref{thm:global_convergence_generic} gives $\|\phi_{t+1}-\phi_t\|_{1,\infty}
\le
M B_{\mathrm f}\eta_{\mathrm f,t}
=
\frac{K}{1-\Gamma}B_{\mathrm f}\eta_{\mathrm s,t},$ and therefore,
\[
\left|
\mathcal L(\pi_t,\phi_t)-\mathcal L(\pi_t,\phi_{t+1})
\right|
\le
L_{\mathcal L,K}\frac{K}{1-\Gamma}B_{\mathrm f}\eta_{\mathrm s,t}.
\]
Let \(c_K=L_{\mathcal L,K}\{K/(1-\Gamma)\}B_{\mathrm f}\). Then
\[
\mathcal L(\pi_t,\phi_t)-\mathcal L(\pi_t,\phi_{t+1})
\ge
-c_K\eta_{\mathrm s,t}.
\]
Adding and subtracting \(\mathcal L(\pi_t,\phi_t)\),
\[
\begin{aligned}
\Gap_t
&=
\mathcal L(\pi^*,\phi^*)-\mathcal L(\pi_t,\phi_{t+1}) =
\{\mathcal L(\pi^*,\phi^*)-\mathcal L(\pi_t,\phi_t)\}
+
\{\mathcal L(\pi_t,\phi_t)-\mathcal L(\pi_t,\phi_{t+1})\} \\
&\ge
\{\mathcal L(\pi^*,\phi^*)-\mathcal L(\pi_t,\phi_t)\}
-
c_K\eta_{\mathrm s,t}.
\end{aligned}
\]
As Assumption~4 gives $\mathcal L(\pi^*,\phi^*)-\mathcal L(\pi_t,\phi_t)
\ge
\mu\Phi(\pi_t,\phi_t),$ taking expectations,
\[
\bar\Gap_t
\ge
\mu\Delta_t-c_K\eta_{\mathrm s,t}.
\]

\noindent
\textbf{Step 2. Convert the drift into a contraction.}
Substituting the lower bound on \(\bar\Gap_t\) into the drift inequality gives
\[
\begin{aligned}
\Delta_{t+1}
\le
\Delta_t-\eta_{\mathrm s,t}\{\mu\Delta_t-c_K\eta_{\mathrm s,t}\}
+
C_1\eta_{\mathrm s,t}^2=
(1-\mu\eta_{\mathrm s,t})\Delta_t
+
(C_1+c_K)\eta_{\mathrm s,t}^2.
\end{aligned}
\]
Let \(C'=C_1+c_K\). Then
\[
\Delta_{t+1}
\le
(1-\mu\eta_{\mathrm s,t})\Delta_t
+
C'\eta_{\mathrm s,t}^2.
\]

\noindent
\textbf{Step 3. Bound the policy distance.} Choose \(\eta_{\mathrm s,t}=2/\{\mu(t+t_0)\}\), where \(t_0\ge2\) is large enough so that the learning rates stay in the step size range used in the proof of Theorem~\ref{thm:global_convergence_generic}. Let
\[
\nu
\ge
\max\left\{
t_0\Delta_0,
\frac{4C'}{\mu^2}
\right\}.
\]
We prove by induction that \(\Delta_t\le\nu/(t+t_0)\) for every \(t\ge0\). The claim holds at \(t=0\) by the choice of \(\nu\). Suppose it holds at time \(t\), and set \(n=t+t_0\). Since \(n\ge2\),
\[
\begin{aligned}
\Delta_{t+1}
\le
\left(1-\frac{2}{n}\right)\frac{\nu}{n}
+
C'\frac{4}{\mu^2n^2} \le
\frac{\nu(n-2)}{n^2}
+
\frac{\nu}{n^2}
=
\frac{\nu(n-1)}{n^2}.
\end{aligned}
\]
Since \((n-1)/n^2\le1/(n+1)\), we obtain
\[
\Delta_{t+1}
\le
\frac{\nu}{n+1}
=
\frac{\nu}{t+1+t_0}.
\]
Thus \(\Delta_t=O(1/(t+t_0))\).

\noindent
\textbf{Step 4. Sum the performance gaps.}
Return to the drift inequality $\Delta_{t+1}
\le
\Delta_t-\eta_{\mathrm s,t}\bar\Gap_t+C_1\eta_{\mathrm s,t}^2.$ Rearranging, we have
\[
\bar\Gap_t
\le
\frac{\Delta_t-\Delta_{t+1}}{\eta_{\mathrm s,t}}
+
C_1\eta_{\mathrm s,t}.
\]
Using \(\eta_{\mathrm s,t}=2/\{\mu(t+t_0)\}\),
\[
\bar\Gap_t
\le
\frac{\mu(t+t_0)}{2}
(\Delta_t-\Delta_{t+1})
+
\frac{2C_1}{\mu(t+t_0)}.
\]
Summing over \(t=0,\ldots,T-1\),
\[
\begin{aligned}
\sum_{t=0}^{T-1}\bar\Gap_t
&\le
\frac{\mu}{2}
\sum_{t=0}^{T-1}
(t+t_0)(\Delta_t-\Delta_{t+1})
+
\frac{2C_1}{\mu}
\sum_{t=0}^{T-1}
\frac{1}{t+t_0}.
\end{aligned}
\]
The weighted telescoping sum satisfies
\[
\sum_{t=0}^{T-1}
(t+t_0)(\Delta_t-\Delta_{t+1})
=
t_0\Delta_0
+
\sum_{t=1}^{T-1}\Delta_t
-
(T-1+t_0)\Delta_T.
\]
Since \(\Delta_T\ge0\),
\[
\sum_{t=0}^{T-1}
(t+t_0)(\Delta_t-\Delta_{t+1})
\le
t_0\Delta_0
+
\sum_{t=1}^{T-1}\Delta_t.
\]
Using the bound from Step 3,
\[
\sum_{t=1}^{T-1}\Delta_t
\le
\nu
\sum_{t=1}^{T-1}
\frac{1}{t+t_0}.
\]
Therefore,
\[
\sum_{t=0}^{T-1}\bar\Gap_t
\le
\frac{\mu t_0}{2}\Delta_0
+
\left(
\frac{\mu\nu}{2}
+
\frac{2C_1}{\mu}
\right)
\sum_{t=0}^{T-1}
\frac{1}{t+t_0}.
\]
The harmonic sum is bounded by a logarithm. Since \(t_0\) is fixed with respect to \(T\), there is a finite constant \(C_2\), independent of \(T\), such that for every \(T\ge2\),
\[
\sum_{t=0}^{T-1}\bar\Gap_t
\le
\frac{C_2\log T}{\mu}.
\]
Dividing by \(T\) gives
\[
\frac{1}{T}
\sum_{t=0}^{T-1}
\E[\Gap_t]
\le
\frac{C_2\log T}{\mu T}.
\]
This proves the accelerated average-gap bound.
\end{proof}


\section{Used-Car Case Study Details}
\label{app:case_study_details}

This section reports implementation details for the used-car inventory--pricing case study in the main text.

\subsection{Simulation Environment Properties}

Table~\ref{tab:customer-attributes} summarizes the customer attributes used in the simulation model. Preferred class and budget capture heterogeneity in vehicle fit and affordability, while price sensitivity and urgency determine how customers respond to prices and time pressure. Seasonal adjustments allow the customer population to evolve over operating periods, making demand nonstationary within the simulation experiments.

\begin{table}[htbp]
\centering
\caption{Customer attributes in the demand model.}
\label{tab:customer-attributes}
\small
\begin{tabular}{p{3cm} p{6cm} p{6.5cm}}
\hline
\textbf{Attribute} & \textbf{Distribution / Definition} & \textbf{Interpretation} \\
\hline
Preferred class &
Categorical over budget, mid-range, and premium, with seasonally varying class shares &
Captures differences in customers' vehicle-class preferences.\\

Budget $b_i$ &
$b_i \sim \mathrm{Uniform}(B_c^{\min}, B_c^{\max})$, with class-specific ranges of \$8{,}000--\$16{,}000 for budget, \$15{,}000--\$28{,}000 for mid-range, and \$25{,}000--\$50{,}000 for premium &
Represents affordability constraints that vary by customer segment and vehicle class. \\

Price sensitivity $\beta_i$ &
Baseline sensitivity $\beta_i^{(0)} \sim \mathrm{Beta}(1.5,1.5)$, adjusted seasonally as
\[
\beta_i=\mathrm{clip}\!\left(\beta_i^{(0)}-0.15\sin(\vartheta_t),\,0,\,1\right)
\]
&
Controls how strongly purchase propensity declines as price rises above the reference level. Customers become less price-sensitive during peak-season periods and more price-sensitive during low-season periods. \\

Urgency $\zeta_i$ &
Baseline urgency $\zeta_i^{(0)} \sim \mathrm{Beta}(2,5)$, adjusted seasonally as
\[
\zeta_i=\mathrm{clip}\!\left(\zeta_i^{(0)}+0.10\sin(\vartheta_t),\,0,\,1\right)
\]
&
Captures time pressure or immediacy of need. Higher urgency increases purchase propensity at a given price. \\
\hline
\end{tabular}
\end{table}

We translate these attributes into purchase probabilities using a reduced-form binary logit specification, consistent with logit-based demand models in discrete-choice and inventory-pricing research \citep{PerezLopezOspina2022,MezaLatorreBonacicLopezOspinaPerez2024}. The utility index combines class fit, reference-price sensitivity, an asymmetric penalty when the posted price exceeds the customer's budget, and urgency. The reference-price component is motivated by prior research on reference-dependent choice \citep{CaputoLuskNayga2020,WangVeldmanTeunter2025}. The probability that customer $i$ purchases vehicle class $c$ at posted price $p_c$ is
\begin{equation}
\label{eq:purchase-prob}
\begin{aligned}
P(\mathrm{buy})
=
\sigma\Bigg(
    \alpha_0 + \alpha_{\mathrm{fit}}
    - \beta_{\mathrm{eff}}\frac{p_c-p_{\mathrm{ref}}}{p_{\mathrm{ref}}}
    - \gamma\max\!\left\{0,\frac{p_c-b_i}{p_{\mathrm{ref}}}\right\}
    + \delta \zeta_i
\Bigg),
\end{aligned}
\end{equation}
where $\sigma(\cdot)$ is the logistic sigmoid, $\alpha_0$ is a baseline intercept, $\alpha_{\mathrm{fit}}$ captures match quality between the customer's preferred class and the offered vehicle class, $\beta_{\mathrm{eff}}=\beta_i(1-\chi\zeta_i)$ adjusts price sensitivity downward as urgency increases, $\gamma$ penalizes prices exceeding the customer's budget $b_i$, $\delta$ captures the direct positive effect of urgency on purchase likelihood, and $p_{\mathrm{ref}}$ is the segment-specific reference price used to normalize deviations. These coefficients are modeling assumptions and are used across all policy comparisons. 

Equation~\eqref{eq:purchase-prob} captures four effects: a preference-match bonus when the offer aligns with the customer's preferred class \citep{PerezLopezOspina2022}; a reference-price penalty when the posted price exceeds $p_{\mathrm{ref}}$ \citep{CaputoLuskNayga2020,WangVeldmanTeunter2025}; a soft affordability penalty when the posted price exceeds $b_i$ \citep{MezaLatorreBonacicLopezOspinaPerez2024}; and an urgency effect that increases transaction propensity while compressing price sensitivity \citep{denBoerKeskin2022,XuDzeverZhao2023}. The urgency--sensitivity interaction is a reduced-form specification consistent with empirical evidence that time-pressured buyers are less price-responsive \citep{Ngo2024,SilalahiPhuongTedjakusumaEunikeRiantama2025}. We treat $\chi$ as a calibrated modeling parameter.

\subsection{Market Primitives and Customer Demand}
\label{app:used_car_demand}

The retailer sells three vehicle classes $\mathcal C=\{\mathrm{budget},\mathrm{mid},\mathrm{premium}\}$. Each class $c$ has acquisition cost $w_c$, per-period holding cost $h_c$, and admissible selling-price range $[p_c^{\min},p_c^{\max}]$. The baseline parameterization is summarized in Table~\ref{tab:vehicle_class_parameters}.

\begin{table}[htbp]
\centering
\caption{Vehicle class parameters used in the case study.}
\label{tab:vehicle_class_parameters}
\small
\begin{tabular}{@{}lccc@{}}
\toprule
\textbf{Class} & \textbf{Acquisition cost $w_c$}
& \textbf{Holding cost $h_c$ / period}
& \textbf{Price range $[p_c^{\min},p_c^{\max}]$} \\
\midrule
Budget    & \$8{,}000  & \$200 & [\$10{,}000, \$15{,}000] \\
Mid-range & \$15{,}000 & \$400 & [\$18{,}000, \$25{,}000] \\
Premium   & \$25{,}000 & \$600 & [\$30{,}000, \$40{,}000] \\
\bottomrule
\end{tabular}
\end{table}

\paragraph{Seasonal arrivals.}
Customer volume follows a seasonal arrival process with a 52-period cycle. For each class $c$, define
\[
m_{c,t}
=
\max\Bigl\{
\varepsilon,
1+a_1^c\sin(\vartheta_t)+b_1^c\cos(\vartheta_t)
+a_2^c\sin(2\vartheta_t)+b_2^c\cos(2\vartheta_t)
\Bigr\},
\qquad
\vartheta_t=\frac{2\pi t}{52},
\]
where $\varepsilon>0$ prevents degenerate zero-demand periods. Let $\pi_c^0$ be the baseline class-share vector. The aggregate seasonal multiplier is $\bar m_t=\sum_{c\in\mathcal C}\pi_c^0m_{c,t}$. Given a base arrival rate $\bar N$, the number of customers in period $t$ is $N_t=\max\{N_{\min},\lfloor \bar N\bar m_t\rceil\}$. Conditional on $N_t$, customer preferred classes are sampled with probabilities proportional to $\pi_c^0m_{c,t}$ and then normalized across classes.

\paragraph{Customer attributes.}
Each arriving customer $i$ has a preferred class $c_i$, budget $b_i$, price-sensitivity score $\beta_i$, and urgency score $\zeta_i$. Budgets are sampled from class-specific ranges:
\[
b_i\sim\mathrm{Uniform}(B_{c_i}^{\min},B_{c_i}^{\max}),
\]
with default ranges \$8{,}000--\$16{,}000 for budget customers, \$15{,}000--\$28{,}000 for mid-range customers, and \$25{,}000--\$50{,}000 for premium customers. Price sensitivity is generated from a Beta-distributed baseline score and mapped to a calibrated sensitivity range:
\[
\tilde\beta_i\sim\mathrm{Beta}(a_\beta,b_\beta),
\qquad
\beta_i = \beta_{\min} + (\beta_{\max}-\beta_{\min}) \mathrm{clip}\{\tilde\beta_i-0.15\sin(\vartheta_t),0,1\}.
\]
Urgency is generated as
\[
\tilde \zeta_i\sim\mathrm{Beta}(2,5), \qquad \zeta_i=\mathrm{clip}\{\tilde \zeta_i+0.10\sin(\vartheta_t),0,1\}.
\]
Thus, customers are less price-sensitive and more urgent during high-season periods, and more price-sensitive and less urgent during low-season periods.

\paragraph{Purchase probability.}
Purchase probabilities follow Equation~\eqref{eq:purchase-prob}. If the preferred class is unavailable, the customer may substitute toward available nearby classes, with stronger substitution between adjacent classes. The realized purchase is generated by comparing the purchase probability against an exogenous uniform random draw.

\subsection{Inventory Pipeline, Profit, and Rewards}
\label{app:used_car_inventory_rewards}

\paragraph{Inventory dynamics.}
At the beginning of period $t$, the replenishment policy selects a class-specific target inventory position $S_{c,t}$. The inventory position is
\[
\mathrm{IP}_{c,t}=I_{c,t}^{\mathrm{oh}}+I_{c,t}^{\mathrm{oo}},
\]
where $I_{c,t}^{\mathrm{oh}}$ is on-hand inventory and $I_{c,t}^{\mathrm{oo}}$ is in-transit inventory. The order quantity is
\[
q_{c,t}=\max\{0,S_{c,t}-\mathrm{IP}_{c,t}\}.
\]
Orders enter a replenishment pipeline and arrive after lead time $L$. Until delivery, in-transit inventory is observed by the long-term policy through $\mathrm{IP}_{c,t}$ but is not available for sale.

Within each period, events unfold in the following order:
\begin{enumerate}[leftmargin=*,nosep]
\item Orders scheduled to arrive in period $t$ are added to on-hand inventory.
\item The long-term replenishment policy observes the period-level state and selects targets $S_{c,t}$.
\item Replenishment orders $q_{c,t}$ enter the pipeline.
\item Customers arrive sequentially. For each arrival, the short-term pricing policy posts prices, the demand model determines whether a sale occurs, and completed sales deplete on-hand inventory.
\item Holding costs, ordering costs, and lost-sale penalties are assessed at the end of the period.
\end{enumerate}
This sequence makes the long-term policy commit to replenishment before observing realized demand in the current period. The long-term policy therefore learns to anticipate near-term demand from lagged signals such as prior sales, conversion rates, and the seasonal demand indicator, rather than reacting to current-period demand directly.

\paragraph{Profit.}
Let $\mathcal S_{c,t}$ be the set of sold class-$c$ vehicles in period $t$, and let $p_i$ be the transaction price of sold vehicle $i$. Period profit is
\[
\Pi_t
=
\sum_{c\in\mathcal C}\sum_{i\in\mathcal S_{c,t}}(p_i-w_c)
-
\sum_{c\in\mathcal C}h_c I^+_{c,t}
-
F\mathbf 1\!\left\{\sum_{c\in\mathcal C}q_{c,t}>0\right\}
-
\sum_{c\in\mathcal C}\lambda_c\ell_{c,t}.
\]
Here $I^+_{c,t}$ is post-sales on-hand inventory, $F$ is the fixed ordering cost incurred whenever at least one order is placed, and $\lambda_c\ell_{c,t}$ is the lost-sale penalty for unserved class-$c$ demand. Acquisition cost is charged when a vehicle is sold rather than when it is ordered, which keeps realized period profit aligned with realized sales and avoids large procurement outflows dominating the learning signal before the corresponding inventory can generate revenue.

\paragraph{Long-term-policy reward.}
The long-term policy receives normalized period profit, 
$r_t^{\mathrm{LT}}=\Pi_t/\kappa$, where $\kappa$ is a numerical scaling constant. The normalization is used only for PPO training stability. All reported results use unnormalized dollar-denominated profit $\Pi_t$.

\paragraph{Short-term-policy reward.}
The short-term policy acts at the customer-arrival timescale. Its dense shaping reward for customer interaction $n$ is
\[
r_n^{\mathrm{ST}}
=
(p_c-w_c)\mathbf 1\{\mathrm{sale}\}
-
\lambda_I\sum_{c\in\mathcal C}\frac{h_c}{N_t}I_c
-
\lambda_{\mathrm{lost}}\mathbf 1\{\mathrm{stockout}\}.
\]

The first term rewards realized margin, the second amortizes period holding cost into a customer-arrival inventory-pressure signal, and the third penalizes failed sales when the relevant class is unavailable. This reward is not added to the long-term reward. It is a dense shaping signal used to guide within-period pricing, while the reported economic outcome remains period profit $\Pi_t$.

\subsection{Factorial Design and Policy Contrasts}
\label{sec:supp-factorial}

Using the four policies, we interpret performance differences through a $2\times2$ design over the replenishment and pricing layers. This design links each comparison to the value of RL pricing, the value of RL replenishment, and the additional value created when both layers are learned jointly.

Let \(P(\cdot)\) denote mean profit over a given analysis window. 
The average gain from RL pricing, averaging over the two replenishment policies, is
\[
\Delta_{\mathrm{pricing}}
=
\frac{1}{2}
\left[
P(\mathrm{OUL{+}RL})-P(\mathrm{OUL{+}Fixed})
+
P(\mathrm{HRL})-P(\mathrm{RL{+}Fixed})
\right].
\]

The average gain from RL replenishment, averaging over the two pricing policies, is
\[
\Delta_{\mathrm{replen}}
=
\frac{1}{2}
\left[
P(\mathrm{RL{+}Fixed})-P(\mathrm{OUL{+}Fixed})
+
P(\mathrm{HRL})-P(\mathrm{OUL{+}RL})
\right].
\]

The interaction contrast is
\[
\Delta_{\mathrm{coord}}
=
P(\mathrm{HRL})
-P(\mathrm{OUL{+}RL})
-P(\mathrm{RL{+}Fixed})
+P(\mathrm{OUL{+}Fixed}).
\]
A positive \(\Delta_{\mathrm{coord}}\) indicates that the gain from adapting one layer is larger when the other layer is also adaptive.

\subsection{Shock Schedules and Common Random Numbers}
\label{app:used_car_shocks_crn}

The experiments combine two disruption channels, supply and demand. Each channel can be inactive or active, giving the four shock settings in the main paper: no shocks, supply-only shocks, demand-only shocks, and joint demand--supply shocks.

\paragraph{Supply shocks.}
During a supply shock, fulfilled orders are reduced before entering the pipeline. In the fractional-cap version,
\[
q_{c,t}^{\mathrm{ful}}=\left\lfloor f_{c,t}q_{c,t}\right\rfloor,
\qquad
f_{c,t}\in(0,1].
\]
The symmetric version uses a shared fulfillment fraction across classes. The asymmetric version assigns class-specific fulfillment fractions, for example from a fixed set such as $\{0.1,0.5,0.8\}$, and permutes the assignment across shock events. This creates uneven scarcity across vehicle classes. Supply shocks may also increase lead times by drawing realized lead times from a wider shock-specific range.

\paragraph{Demand shocks.}
During a demand shock, the customer count in period $t$ is multiplied by an event-specific demand factor $m_{d,t}$:
\[
N_t^{\mathrm{shock}}=\max\{N_{\min},\lfloor m_{d,t}N_t\rceil\}.
\]
Values $m_{d,t}>1$ represent demand surges, and values $m_{d,t}<1$ represent demand contractions. Demand shocks also reweight the class mix of arrivals, so different events emphasize different vehicle segments.

\paragraph{Timing and observability.}
Shocks are exogenous and are not observed before they occur. Once a shock is realized, its effects become visible through the operating state. Supply shocks appear through realized fulfillment, pipeline delays, and inventory availability. Demand shocks appear through realized customer volume, class mix, and conversion behavior. Thus, the first affected period cannot be anticipated mechanically, but later decisions can adapt to the realized operating condition.

\paragraph{Common random numbers.}
All policies are evaluated using common random numbers. For each seed, the simulator generates the same customer arrivals, customer attributes, purchase draws, shock realizations, and fulfillment randomness across all four policies. The only difference is the policy. This paired design reduces sampling noise and supports within-seed comparisons of period profit, disruption response, and coordination effects.

\subsection{Policies and Training Protocol}
\label{app:used_car_policy_protocol}

\paragraph{OUL+Fixed.}
The non-learning benchmark combines an order-up-to-level replenishment rule with fixed-markup pricing. For each class $c$, the OUL target is computed from estimated period demand, lead time, and a target service level:
\[
S_{c,t}^{\mathrm{OUL}}
=
\widehat D_{c,t}(R+\widehat L)+z_\alpha\widehat\sigma_{c,t}\sqrt{R+\widehat L},
\]
where $R$ is the review period, $\widehat D_{c,t}$ is a moving-average estimate of demand per period, $\widehat\sigma_{c,t}$ is the corresponding demand-variability estimate, and $z_\alpha$ is the safety factor for the target service level. Fixed prices are set at a 30\% markup over acquisition cost and clipped to the feasible price range.

\paragraph{OUL+RL.}
OUL+RL keeps the OUL replenishment rule but replaces fixed pricing with the short-term PPO policy. It identifies the value of adaptive customer-arrival pricing when replenishment remains rule-based.

\paragraph{RL+Fixed.}
RL+Fixed replaces the OUL replenishment rule with the long-term PPO policy but keeps fixed-markup pricing. It identifies the value of adaptive replenishment when pricing is fixed.

\paragraph{HRL.}
HRL learns both layers. The long-term policy selects period-level base-stock targets $S_{c,t}$. The short-term policy observes the current inventory state and arriving-customer attributes and sets class-level prices within the feasible price ranges. Assortment is derived deterministically from available inventory and posted margins. In-stock classes are ranked by $p_c-w_c$, and the highest-ranked feasible offers are shown to the customer. This keeps the short-term action as pricing while allowing inventory scarcity to affect the offered set.

\paragraph{Evaluation windows.}
Each scenario is evaluated over 30 independent replications. Experiment~1 uses 3{,}000 operating-period runs for no shocks, demand-only shocks, supply-only shocks, and joint shocks. Experiment~2 uses 5{,}500 operating-period runs and introduces a prolonged high-intensity disruption cluster over periods 3{,}000--3{,}300. Unless stated otherwise, reported economic comparisons use the post-initialization window beginning at period 450.

\paragraph{Training stabilization.}
The HRL training schedule uses an initialization phase before the main evaluation window. The first 350 periods stabilize the short-term pricing layer. The next 100 periods allow the long-term layer to adapt against the stabilized short-term policy. Full joint learning begins at period 450. The same post-initialization convention is used for the main economic comparisons across policies, so reported results focus on substantive operating behavior rather than early training transients.

\paragraph{Statistical reporting.}
Tables report seed-level means with 95\% confidence intervals unless otherwise specified. Pairwise comparisons are based on within-seed differences, using the common-random-number structure. Time-series plots report rolling means across replications, with confidence bands summarizing cross-seed variation. The coordination contrast reported in the main text is
\[
\Delta_{\mathrm{coord}}
=
P(\mathrm{HRL})-P(\mathrm{OUL{+}RL})-P(\mathrm{RL{+}Fixed})+P(\mathrm{OUL{+}Fixed}),
\]
computed over the relevant full-horizon, shock, or recovery window.

\section{Additional Numerical Illustration}
\label{sec:additional_numerical}

This section presents two controlled numerical illustrations of the convergence mechanisms in Theorems~\ref{thm:global_convergence_generic} and~\ref{thm:rate_improve_average}. The first considers baseline convergence under synchronized updates in a bounded two-timescale environment. The second considers faster learning when poor decisions produce more visible objective losses. These experiments clarify the mechanisms behind the bounds rather than estimate hidden constants or verify asymptotic rates.

\subsection{Environments}
\label{app:exact_env_controlled}

Both environments have finite long-term and short-term action spaces, a fixed within-period horizon (K=20), and use the synchronized finite-state population policy-space update analyzed in the main paper. The long-term state space and both action spaces are \({-2,-1,0,1,2}\). The projection operator \(\Pi\) maps a scalar to the nearest element of this set.

For panel~(a), which illustrates Theorem~\ref{thm:global_convergence_generic}, the environment is a bounded two-timescale system without additional sharpness. The short-term reward is
\[
r(x,u,a)=-0.20\{a-(0.7x+0.5u)\}^2,
\]
the long-term cost is \(c(x,u)=0.40(u-0.5x)^2\), and the next long-term state is
\[
x^{+}=\Pi\{\mathrm{round}(0.6x+0.4\bar a+\varepsilon)\},
\]
where \(\bar a\) is the average within-period short-term action and \(\varepsilon\in\{-1,0,1\}\) has probabilities \(0.2,0.6,0.2\).

For panel~(b), which illustrates Theorem~\ref{thm:rate_improve_average}, the reward and cost create a sharper quadratic landscape around the optimal long-term commitment. Writing \(m(x)=x\), the short-term reward is
\[
r(x,u,a)=-1.30\{a-m(x)\}^2-0.90(a-u)^2,
\]
the long-term cost is \(c(x,u)=1.10\{u-m(x)\}^2\), and the next long-term state is
\[
x^{+}=\Pi\{\mathrm{round}(0.5x+0.5\bar a+\varepsilon)\},
\]
where \(\varepsilon\in\{-1,0,1\}\) has probabilities \(0.1,0.8,0.1\).

\subsection{Oracle Benchmark and Reported Metric}
\label{app:oracle_metric_controlled}

Because the environments are finite and fully specified, the oracle benchmark is computed exactly by finite-state dynamic programming \citep{puterman2014markov}. The oracle is the optimal stationary joint policy \((\pi^*,\phi^*)\). At the long-term level, it chooses the maximizing long-term action. Within each long-term period, it chooses the maximizing short-term action conditional on the long-term action and the remaining continuation value.

The learned policy is the policy-space PPO-style update used in the convergence analysis. At each long-term period \(t\), we evaluate the post-short-term-update pair \((\pi_t,\phi_{t+1})\) using
\[
\mathcal G_t^{\mathrm{num}}
=
\mathcal L(\pi^*,\phi^*)-\mathcal L(\pi_t,\phi_{t+1}).
\]
Figure~1 reports the running average \(t^{-1}\sum_{\tau=0}^{t-1}\mathcal G_{\tau}^{\mathrm{num}}\), averaged over 30 random initializations. Evaluation is exact in the finite environments, and across-initialization variation is negligible at the plotted scale, so confidence bands are omitted for readability.

\subsection{Results}

\begin{figure*}[!htpb]
\centering
\begin{minipage}[t]{0.43\textwidth}
\centering
\includegraphics[width=\textwidth]{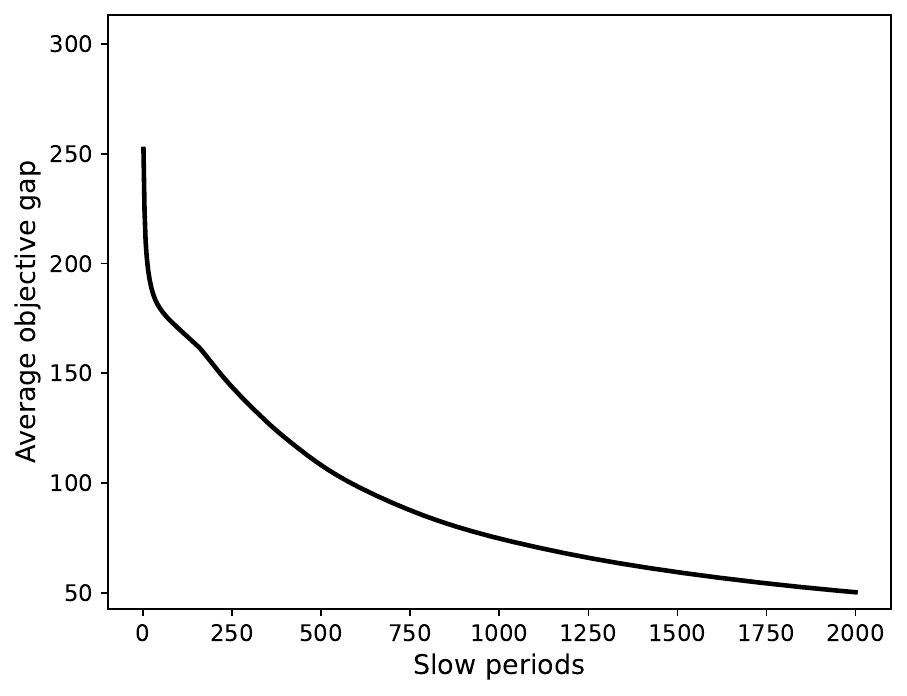}

{\small (a) Theorem~\ref{thm:global_convergence_generic}}

\end{minipage}
\hfill
\begin{minipage}[t]{0.43\textwidth}
\centering
\includegraphics[width=\textwidth]{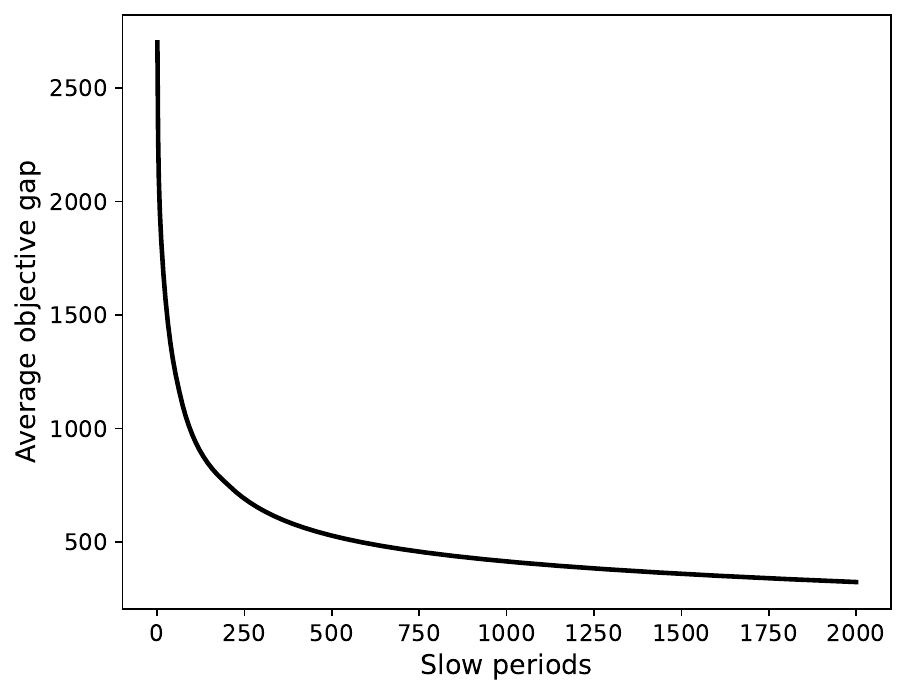}

{\small (b) Theorem~\ref{thm:rate_improve_average}}

\end{minipage}
\caption{Controlled numerical illustrations of Theorems~\ref{thm:global_convergence_generic} and~\ref{thm:rate_improve_average}. Each panel reports the running average post-short-term-update optimality gap, averaged over 30 random initializations.}
\label{fig:theory_controlled_simulations}
\end{figure*}

Figure~\ref{fig:theory_controlled_simulations} reports the running-average gap in the two environments. Panel~(a) uses a bounded two-timescale environment without additional sharpness and shows a steadily decreasing running-average gap, consistent with the baseline mechanism in Theorem~\ref{thm:global_convergence_generic}. Panel~(b) uses a sharper reward-and-cost structure in which poor decisions produce more visible objective losses. Its larger initial gap reflects the stronger penalty for suboptimal decisions rather than better absolute performance. The resulting clearer learning signal supports faster correction, and the running-average gap decreases more rapidly. The panels illustrate the mechanisms underlying Theorems~\ref{thm:global_convergence_generic} and~\ref{thm:rate_improve_average}; they do not estimate hidden constants or verify asymptotic rates.
}

\end{document}